\documentclass[journal,transmag]{IEEEtran}

\usepackage{times}
\usepackage{epsfig}
\usepackage{graphicx}
\usepackage{amsmath,mathrsfs}
\usepackage{amssymb}
\usepackage{enumerate}
\usepackage{threeparttable}
\usepackage{multirow}
\usepackage{booktabs}
\usepackage{algorithm}
\usepackage{algorithmic}
\usepackage{amsmath,bm,amsfonts}
\usepackage{mcode}
\usepackage{subcaption}
\usepackage{amsthm}
\usepackage{url}
\usepackage{caption,color}
\usepackage{threeparttable}

\newcommand{\sgnop}{\operatorname{sgn}}
\newcommand{\sgn}[1]{\ensuremath{\sgnop\left(#1\right)}}

\newcommand{\inproduct}[2]{\left\langle#1,#2\right\rangle}

\newtheorem{thm}{Theorem}

\newtheorem{prop}{Proposition}
\newtheorem{defn}{Definition}
\newcommand{\Tr}{\mcode{Tr}}

\newcommand{\unfold}{\mcode{unfold}}
\newcommand{\fold}{\mcode{fold}}

\newcommand{\st}{{\text{s.t.}}}

\newcommand{\rankm}{\text{rank}}
\newcommand{\rankt}{\text{rank}_{\text{t}}}

\newcommand{\ranka}{\text{rank}_{\text{a}}}

\newcommand{\none}{n_{(1)}}
\newcommand{\ntwo}{n_{(2)}}
\newcommand{\Rn}{\mathbb{R}^{n_1\times n_2\times n_3}}

\newcommand{\nss}{n\times n\times n_3}

\newcommand{\Pomega}{\bm{\mathcal{P}}_{\bm{{\Omega}}}}

\newcommand{\Omegat}{\bm{{\Omega}}}

\newcommand{\bcirc}{\mcode{bcirc}}
\newcommand{\bdiag}{\mcode{bdiag}}

\newcommand{\A}{\bm{\mathcal{A}}}
\newcommand{\B}{\bm{\mathcal{B}}}
\newcommand{\C}{\bm{\mathcal{C}}}
\newcommand{\D}{\bm{\mathcal{D}}}
\newcommand{\E}{\bm{\mathcal{E}}}

\newcommand{\I}{\bm{\mathcal{I}}}
\newcommand{\LL}{\bm{\mathcal{L}}}
\newcommand{\M}{\bm{\mathcal{M}}}
\newcommand{\PP}{\bm{\mathcal{P}}}
\newcommand{\Q}{\bm{\mathcal{Q}}}

\newcommand{\SSS}{\bm{\mathcal{S}}}

\newcommand{\U}{\bm{\mathcal{U}}}
\newcommand{\V}{\bm{\mathcal{V}}}

\newcommand{\X}{\bm{\mathcal{X}}}
\newcommand{\Y}{\bm{\mathcal{Y}}}

\newcommand{\Lhat}{\hat{\bm{\mathcal{L}}}}
\newcommand{\Shat}{\hat{\bm{\mathcal{S}}}}

\newcommand{\Abar}{\bm{\mathcal{\bar{A}}}}
\newcommand{\Bbar}{\bm{\mathcal{\bar{B}}}}
\newcommand{\Cbar}{\bm{\mathcal{\bar{C}}}}

\newcommand{\Ubar}{\bm{\mathcal{\bar{U}}}}
\newcommand{\Vbar}{\bm{\mathcal{\bar{V}}}}
\newcommand{\Sbar}{\bm{\mathcal{\bar{S}}}}

\newcommand{\Ibar}{\bm{\mathcal{\bar{I}}}}

\newcommand{\Am}{\bm{{A}}}
\newcommand{\Bm}{\bm{{B}}}

\newcommand{\Dm}{\bm{{D}}}
\newcommand{\Em}{\bm{E}}
\newcommand{\F}{\bm{{F}}}

\newcommand{\Lm}{\bm{L}}

\newcommand{\Xm}{\bm{X}}

\renewcommand{\Im}{\bm{I}}

\newcommand{\Ambar}{\bm{\bar{A}}}
\newcommand{\Bmbar}{\bm{\bar{{B}}}}
\newcommand{\Cmbar}{\bm{\bar{C}}}

\newcommand{\Umbar}{\bm{\bar{U}}}
\newcommand{\Vmbar}{\bm{\bar{V}}}
\newcommand{\Xmbar}{\bm{\bar{X}}}

\newcommand{\Ymbar}{\bm{\bar{Y}}}

\newcommand{\Smbar}{\bm{\bar{S}}}

\newcommand{\Imbar}{\bm{\bar{I}}}

\newcommand{\Bmtilde}{\bm{\tilde{{B}}}}

\renewcommand{\aa}{{\bm{a}}}
\newcommand{\e}{{\text{e}}}
\newcommand{\vv}{{\bm{v}}}

\newcommand{\norm}[1]{\lVert#1\rVert}

{%
\begin{list}{#1}{
\vspace{-\topsep}
\vspace{-\partopsep}
\setlength{\itemindent}{0cm}
\setlength{\rightmargin}{0cm}
\setlength{\listparindent}{0cm}
\settowidth{\labelwidth}{#1}
\setlength{\leftmargin}{\labelwidth}
\addtolength{\leftmargin}{\labelsep}
\setlength{\itemsep}{0cm}
}%
}%
{%
\end{list}
\vspace{-\topsep}
\vspace{-\partopsep}
}

{\begin{enumerate}%
\renewcommand{\theenumi}{\roman{enumi}}%
}%
{\end{enumerate}}

\newcommand{\sumi}{\sum_{i=1}^{n_3}}

\begin{document}

\title{Exact Recovery of Tensor Robust Principal Component Analysis under Linear Transforms}
	
\author{Canyi Lu~and~Pan Zhou
\thanks{C. Lu is with the Department of Electrical and Computer Engineering, Carnegie Mellon University (e-mail: canyilu@gmail.com).}
\thanks{P. Zhou is with the Department of Electrical and Computer Engineering, National University of Singapore, Singapore (e-mail: panzhou3@gmail.com).}
}

\maketitle
\vspace{-100em}
\begin{abstract}

This work studies  the Tensor Robust Principal Component Analysis (TRPCA) problem, which aims to exactly recover the low-rank and sparse components from their sum. Our model is motivated by the recently proposed  linear transforms based tensor-tensor product and tensor SVD. We define a new transforms depended tensor rank and the corresponding tensor nuclear norm. Then we solve the TRPCA problem by convex optimization whose objective is a weighted combination of the new tensor nuclear norm and the $\ell_1$-norm. In theory, we show that under certain incoherence conditions, the convex program exactly recovers the underlying low-rank and sparse components. It is of great interest that our new TRPCA model generalizes existing works. In particular, if the studied tensor reduces to a matrix, our TRPCA model reduces to the known matrix RPCA \cite{RPCA}. Our new TRPCA which is allowed to use general linear transforms can be regarded as an extension of our former TRPCA work~\cite{lu2016tensorrpca} which uses the  discrete Fourier transform. But their proof of the recovery guarantee is different. Numerical experiments verify our results and the application on image recovery demonstrates the superiority of our method.
\end{abstract}

\section{Introduction}

Tensors are the higher-order generalization of vectors and matrices. They have many
applications in the physical, imaging and information sciences, and an in depth
survey can be found in \cite{kolda2009tensor}. Tensor decompositions give a concise representation
of the underlying structure of the tensor, revealing when the tensor-data can be
modeled as lying close to a low-dimensional subspace. Having originated in the fields of psychometrics and
chemometrics, these decompositions are now widely used in other application areas
such as computer vision \cite{Vasilescu2002Multilinear}, web data mining \cite{franz2009triplerank}, and signal processing \cite{sidiropoulos2000parallel}.

Tensor decomposition faces several challenges: arbitrary outliers, missing
data/partial observations, and computational efficiency. Tensor decomposition resembles principal component analysis (PCA) for matrices in many ways. 
The two commonly used decompositions are the CANDECOMP/PARAFAC (CP) and Tucker decomposition \cite{kolda2009tensor}. It is well known that PCA is sensitive to outliers and gross corruptions (non-Gaussian noise). 
Since the CP and Tucker decompositions are also based on least-squares approximation, they are prone to these problems as well. Algorithms based on nonconvex
formulations have been proposed to robustify tensor decompositions against outliers \cite{engelen2009a} and missing data~\cite{Acar09scalabletensor}. However, they suffer from the lack of global optimality guarantees and statistical guarantee.

In this work, we study the Tensor Robust Principal Component Analysis (TRPCA) problem which aims to find the low-rank and sparse tensors decomposition from their sum. More specifically, assume that a  tensor $\X$ can be decomposed as $\X=\LL_0+\E_0$ where $\LL_0$ is a low-rank tensor and $\E_0$ is a sparse tensor. TRPCA aims to recover $\LL_0$ and $\E_0$ from $\X$. We focus on the convex model which can be solved exactly and efficiently, and the solutions own the theoretical guarantee. TRPCA extends the well known Robust PCA \cite{RPCA} model, i.e.,
\begin{equation}\label{rpca}
\min_{\Lm,\Em} \ \norm{\Lm}_*+\lambda\norm{\Em}_1, \ \st  \ \Xm=\Lm+\Em,
\end{equation}
where $\norm{\Lm}_*$ denotes the matrix nuclear norm, $\norm{\Em}_1$ denotes the $\ell_1$-norm and $\lambda>0$. Under certain incoherence conditions, it is proved in \cite{RPCA} that the solution to (\ref{rpca}) exactly recovers the underlying low-rank and sparse components from their sum. RPCA has many applications in image and video analysis~\cite{ji2011robust,peng2012rasl}.

It is natural to consider the tensor extension of RPCA. Unlike the matrix cases, the recovery theory for low-rank tensor estimation problems is far from being well established. The main issues lie in the definitions of tensor rank. There have different formulations of tensor SVD decompositions which associated with different tensor rank definitions. Then they lead to different tensor RPCA model. Compared with the matrix RPCA, the tensor RPCA models have several limitations.
The tensor CP rank is defined as the smallest number of rank one tensor CP decomposition. However, it is NP-hard to compute. Moreover, finding the best convex relaxation of the tensor CP rank is also NP-hard~\cite{hillar2013most}. Unlike the matrix case, where the convex relaxation of the rank, viz., the nuclear norm, can be computed efficiently. These issues make the low CP rank tensor recovery problem challenging to solve. The tensor Tucker rank is more widely studied as it is tractable. The sum-of-nuclear-norms (SNN) is served as a simple convex surrogate for tensor Tucker rank. This idea, after first being proposed in \cite{liu2013tensor}, has been studied in \cite{gandy2011tensor,tomioka2011statistical}, and successfully
applied to various problems \cite{kreimer2013nuclear,signoretto2011tensor}. Unlike the matrix cases, the recovery theory for low Tucker rank tensor estimation problems is far from being well established. The work \cite{tomioka2011statistical} conducts a  statistical analysis of tensor decomposition and provides the first theoretical guarantee for SNN minimization. This result has been further enhanced in recent works~\cite{mu2013square, huang2014provable}. The main limitation of SNN minimization is that it is suboptimal, e.g., for tensor completion, the required sample complexity is much higher than the degrees of freedom. This is different from the matrix nuclear norm minimization which leads to order optimal sampling complexity \cite{candes2009exact}. Intuitively, the limitation of SNN lies in the issue that SNN is not a tight convex relaxation of tensor Tucker rank \cite{romera2013new}.

More recently, based on the tensor-tensor product (t-product) and tensor SVD (t-SVD) \cite{kilmer2011factorization}, a new tensor tubal rank and Tensor Nuclear Norm (TNN) is proposed and applied to tensor completion \cite{zhang2014novel,semerci2014tensor,lu2018exact} and tensor robust PCA~\cite{lu2016tensorrpca,lu2018tensor}. Compared with the Tucker rank based SNN model, the main advantage of the t-product induced TNN based models is that they own the same tight recovery bound as the matrix cases~\cite{chen2013incoherence,zhang2015exact,lu2018tensor,lu2016tensorrpca}. Also, unlike the CP rank, the tubal rank and TNN are computable.  In \cite{kernfeld2015tensor}, the authors observe that the t-product is based on a convolution-like operation, which can be implemented using the Discrete Fourier Transform (DFT). In order to properly motivate this transform based approach, a more general tensor-tensor product definition is proposed based on any invertible linear transforms. The transforms based t-product also owns a matrix-algebra based interpretation in the spirit of \cite{gleich2013power}. Such a new transforms based t-product is of great interest in practice as it allows to use different linear transforms for different real data.

In this work, we propose a more general tensor RPCA model based on a new tensor rank and tensor nuclear norm induced by invertible linear transforms, and provide the theoretical recovery guarantee. We show that when the invertible linear transform given by the matrix $\Lm$ further satisfies
\begin{align*}
\Lm^\top \Lm = \Lm\Lm^\top = \ell  \Im,
\end{align*}
for some constant $\ell>0$, a new  tensor tubal rank and tensor nuclear norm can be defined induced by the transform based t-product. Equipped with the new tensor
nuclear norm, we then solve the TRPCA problem by solving a convex program (see the definitions of the notations in Section \ref{sec_notations})
\begin{align}\label{trpca}
\min_{\LL,\E} \ \norm{\LL}_*+\lambda\norm{\SSS}_1, \ \st \ \X=\LL+\SSS.
\end{align}
In theory, we prove that, under certain incoherence conditions, the solution to (\ref{trpca}) exactly recovers the underlying low-rank $\LL_0$ and sparse $\SSS_0$ components with high probability. Interestingly, our model is much more flexible and the recovery result is much more general than existing works~\cite{RPCA,lu2016tensorrpca}, since we have much more general choices of the invertible linear transforms. If the tensors reduce to matrices, our model and main result reduce to the RPCA cases in~\cite{RPCA}. If the discrete Fourier transform is used as the linear transform $\Lm$, the t-product, tensor nuclear norm and our TRPCA model reduce to the cases in our former work~\cite{lu2016tensorrpca,lu2018tensor}. However, our extension on more general choices of linear transforms is non-trivial, since the proofs of the exact recovery guarantee is very different due to different property of the linear transforms. We will give detailed discussions in Section \ref{sec_rpca}.


The rest of this paper is structured as follows. Section~\ref{sec_notations} gives some notations and presents the new tensor nuclear norm induced by the t-product under linear transforms. Section \ref{sec_rpca} presents the theoretical guarantee of the new tensor nuclear norm based convex TRPCA model.  Numerical experiments conducted on both synthesis and real world data are presented in Section \ref{sec_exp}. We finally conclude this work in Section \ref{sec_con}.

\section{Tensor Tubal Rank and Tensor Nuclear Norm under Linear Transform}\label{sec_notations}

\subsection{Notations}
We first introduce some notations and definitions used in this paper.  We follow similar notations as in \cite{lu2018tensor}.
We denote scalars by lowercase letters, {e.g.}, $a$, vector by boldface lowercase letters, {e.g.}, $\aa$, matrices by boldface capital letters, {e.g.}, $\Am$, and tensors by boldface Euler script letters, {e.g.}, $\A$.  {For a 3-way tensor $\A\in\mathbb{R}^{n_1\times n_2\times n_3}$, we denote its $(i,j,k)$-th entry as $\A_{ijk}$ or $a_{ijk}$ and use the Matlab notation $\A(i,:,:)$, $\A(:,i,:)$ and $\A(:,:,i)$ to denote respectively the $i$-th horizontal, lateral and frontal slice \cite{kolda2009tensor}.} More often, the frontal slice $\A(:,:,i)$ is denoted compactly as $\Am^{(i)}$. The tube is denoted as $\A(i,j,:)$. The inner product between $\Am$ and $\Bm$ in $\mathbb{R}^{n_1\times n_2}$ is defined as $\inproduct{\Am}{\Bm}=\Tr(\Am^\top\Bm)$, where $\Am^\top$ denotes the transpose of $\Am$ and $\Tr(\cdot)$ denotes the matrix trace. The inner product between $\A$ and $\B$ in $\Rn$ is defined as $\langle\A,\B\rangle=\sum_{i=1}^{n_3}\inproduct{\Am^{(i)}}{\Bm^{(i)}}$. We denote $\bm I_n$ as the $n\times n$ sized identity matrix.



Some norms of vector, matrix and tensor are used. We denote the Frobenius norm as $\norm{\A}_F=\sqrt{\sum_{ijk}a_{ijk}^2}$, the $\ell_1$-norm as $\norm{\A}_1=\sum_{ijk}|a_{ijk}|$, and the infinity norm as $\norm{\A}_\infty=\max_{ijk}|a_{ijk}|$, respectively. The above norms reduce to the vector or matrix norms if $\A$ is a vector or a matrix. For $\vv\in\mathbb{R}^n$, the $\ell_2$-norm is $\norm{\vv}_2 = \sqrt{\sum_{i}v_{i}^2}$. The spectral norm of a matrix $\Am$ is denoted as $\norm{\Am} = \max_{i}\sigma_i(\Am)$, where $\sigma_i(\Am)$'s are the singular values of $\Am$. The matrix nuclear norm is $\norm{\Am}_* = \sum_{i}\sigma_i(\Am)$.

\subsection{T-product Induced Tensor Nuclear Norm under Linear Transform}
We first give some notatons, concepts and properties about the tensor-tensor product proposed in \cite{kilmer2011factorization}.
For $\A\in\Rn$, we define the block circulant matrix ${\bcirc}(\A)\in\mathbb{R}^{n_1n_3\times n_2n_3}$ of $\A$ as
\begin{align*} 
{\bcirc}(\A) =
\begin{bmatrix}
\Am^{(1)} &\Am^{(n_3)} &\cdots &\Am^{(2)} \\
\Am^{(2)} &\Am^{(1)} & \cdots &\Am^{(3)} \\
\vdots & \vdots & \ddots & \vdots \\
\Am^{(n_3)} & \Am^{(n_3-1)} & \cdots & \Am^{(1)}
\end{bmatrix}.
\end{align*}
We define the following two operators
\begin{equation*}
\mcode{unfold}(\A) =
\begin{bmatrix}
\Am^{(1)} \\ \Am^{(2)} \\ \vdots \\ \Am^{(n_3)} 
\end{bmatrix}, \ \mcode{fold}(\mcode{unfold}(\A))=\A,
\end{equation*}
where the $\unfold$ operator maps $\A$ to a matrix of size $n_1n_3\times n_2$ and $\fold$ is its inverse operator. Let $\A\in\Rn$ and $\B\in\mathbb{R}^{n_2\times l\times n_3}$. Then the t-product $\C = \A*\B$ is defined to be a tensor of size $n_1\times l\times n_3$, 
\begin{equation}\label{tproducdef}
\C = \A*\B = \mcode{fold}(\mcode{bcirc}(\A)\cdot\mcode{unfold}(\B)).
\end{equation} 
The block circulant matrix can be block diagonalized using Discrete Fourier Transform (DFT) matrix $\F_{n_3}$, i.e.,
\begin{align}\label{dftpro}
(\F_{n_3} \otimes \bm{I}_{n_1}) \cdot \mcode{bcirc}(\A) \cdot (\F_{n_3}^{-1} \otimes \bm{I}_{n_2}) = \Ambar,
\end{align}
where $\otimes$ denotes the Kronecker product, and $\Ambar$ is a block diagonal matrix with the $i$-th block $\Ambar^{(i)}$ being the $i$-th frontal slices of $\Abar$ obtained by performing DFT on $\A$ along the 3-rd dimension, i.e., 
\begin{align}\label{defLfft}
\Abar =  \A \times_3 \F_{n_3},
\end{align}
where $\times_3$ denotes the mode-3 product (see Definition 2.5 in~\cite{kernfeld2015tensor}), and $\F_{n_3}\in\mathbb{C}^{n_3\times n_3}$ denotes the DFT matrix (see \cite{lu2018tensor} for the formulation). By using the Matlab command $\mcode{fft}$, we also have $\Abar=\mcode{fft}(\A,[\ ],3)$.  We denote 
\begin{align}
\D = \A \odot \B
\end{align}
as the frontal-slice-wise product (Definition 2.1 in \cite{kernfeld2015tensor}), i.e.,
\begin{align}
\Dm^{(i)} = \Am^{(i)} \Bm^{(i)}, \ i=1,\cdots,n_3.
\end{align} 
Then the block diagonalized property in (\ref{dftpro}) implies that $\Cbar = \Abar\odot \Bbar$. So the t-product in (\ref{tproducdef}) is equivalent to the matrix-matrix product in the transform domain using DFT. 

Instead of using the specific discrete Fourier transform, the recent work \cite{kernfeld2015tensor} proposes a more general definition of t-product under any invertible linear transform $L$. In this work, we consider the linear transform 
$L: \mathbb{R}^{n_1\times n_2\times n_3}\rightarrow \mathbb{R}^{n_1\times n_2\times n_3}$ which gives $\Abar$ by performing a linear transform on $\A$ along the 3-rd dimension, i.e.,
\begin{align}\label{defL}
\Abar = L(\A) = \A \times_3 \Lm,
\end{align}
where the linear transform is given by  $\Lm\in\mathbb{R}^{n_3\times n_3}$ which can be arbitrary invertible matrix. 
We also have the inverse mapping
given by
\begin{align}
L^{-1}(\A) = \A\times_3 \Lm^{-1}.
\end{align}
Then the  t-product under linear transform $L$ is defined as follows.
\begin{defn} \textbf{(T-product)} \cite{kernfeld2015tensor}
	Let	$L$ be any invertible linear transform in (\ref{defL}), $\A\in\Rn$ and $\B\in\mathbb{R}^{n_2\times l\times n_3}$.  The  t-product of $\A$ and $\B$ under $L$, denoted as $\C = \A*_L\B$, is defined such that $L(\C) = L(\A) \odot L(\B)$.
\end{defn}

We denote $\Ambar\in\mathbb{R}^{n_1n_3\times n_2n_3}$ as a block diagonal matrix with its $i$-th block on the diagonal as the $i$-th frontal slice $\Ambar^{(i)}$ of ${\Abar} = L(\A)$, {i.e.},
\begin{equation*}\label{eq_Abardef}
\Ambar = \bdiag(\Abar) =
\begin{bmatrix}
\Ambar^{(1)} & & & \\
& \Ambar^{(2)} & & \\
& & \ddots & \\
& & & \Ambar^{(n_3)}
\end{bmatrix},
\end{equation*}
where $\bdiag$ is an operator which maps tensor $\Abar$ to the block diagonal matrix $\Ambar$. Then $L(\C) = L(\A) \odot L(\B)$ is equivalent to $\Cmbar = \Ambar \Bmbar$. This implies that the t-product under $L$ is equivalent to the matrix-matrix product in the transform domain. By using this property, Algorithm \ref{alg_ttprod} gives the way for computing t-product. Figure \ref{fig_transtproduct} gives an intuitive illustration of t-product and its equivalence in the transform domain. It is easy to see that the time complexity for computing $L(\A)$ is $O(n_1n_2n_3^2)$. Then the time complexity for computing  $\A*_L\B$ is $O(n_1n_2n_3^2+n_2ln_3^2+n_1n_2n_3l)$.  But the cost can be lower  if $L$ has special properties. For example, the cost for computing $L(\A)$ is $O(n_1n_2n_3\log n_3)$ for fast Fourier transform \cite{lu2018tensor}.



\begin{algorithm}[!t]
	\caption{T-product under linear transform $L$}
	\textbf{Input:} $\A\in\Rn$, $\B\in\mathbb{R}^{n_2\times l\times n_3}$, and $L: \mathbb{R}^{n_1\times n_2\times n_3}\rightarrow \mathbb{R}^{n_1\times n_2\times n_3}$.\\
	\textbf{Output:} $\C = \A *_L \B\in\mathbb{R}^{n_1\times l\times n_3}$.
	\begin{enumerate}[1.]
		\item Compute $\Abar=L(\A)$ and $\Bbar = L(\B)$.
		\item Compute each frontal slice of $\Cbar$ by
		
		
		\hspace*{0.4cm} $\Cmbar^{(i)} = \Ambar^{(i)} \Bmbar^{(i)}$, $i=1,\cdots,  n_3$
		%
		
		\item Compute $\C=L^{-1}(\Cbar)$.
	\end{enumerate}
	\label{alg_ttprod}	
\end{algorithm}

The t-product enjoys many similar properties as the matrix-matrix product. For example, the t-product is associative, {i.e.}, $\A*_L(\B*_L\C) = (\A*_L\B)*_L\C$.
\begin{defn} \textbf{(Tensor transpose)}  \cite{kernfeld2015tensor}
	Let	$L$ be any invertible linear transform in (\ref{defL}), and $\A\in\mathbb{R}^{n_1\times n_2\times n_3}$. Then the tensor transpose of $\A$ under $L$, denoted as $\A^\top$,  satisfies $L(\A^\top)^{(i)} = (L(\A)^{(i)})^\top$, $i=1,\cdots,n_3$.
\end{defn}

\begin{defn} \textbf{(Identity tensor)} \cite{kernfeld2015tensor}
	Let	$L$ be any invertible linear transform in (\ref{defL}). Let $\I\in\mathbb{R}^{n\times n\times n_3}$ so that each frontal slice of $L(\I) = \Ibar$ is a $n\times n$ sized identity matrix. Then $\I = L^{-1}(\Ibar)$ is called the identity tensor under $L$.
\end{defn}
It is clear that $\A *_L\I = \A$ and $\I *_L \A = \A$ given the appropriate dimensions. The tensor $\Ibar=L(\I)$ is a tensor with each frontal slice being the identity matrix.
\begin{defn} \textbf{(Orthogonal tensor)} \cite{kernfeld2015tensor}
	Let	$L$ be any invertible linear transform in (\ref{defL}). A tensor $\Q\in\mathbb{R}^{n\times n\times n_3}$ is orthogonal under $L$ if it satisfies $\Q^\top *_L \Q=\Q*_L\Q^\top=\I$.
\end{defn}
\begin{defn}\textbf{(F-diagonal Tensor)}
	A tensor is called f-diagonal if each of its frontal slices is a diagonal matrix. 
\end{defn}
\begin{thm} \textbf{(T-SVD)} \cite{kernfeld2015tensor} \label{thmtsvd}
	Let	$L$ be any invertible linear transform in (\ref{defL}), and $\A\in\Rn$. Then it can be factorized as
	\begin{equation}\label{eq_tsvd}
	\A=\U*_L\SSS*_L\V^\top,
	\end{equation}
	where $\U\in \mathbb{R}^{n_1\times n_1\times n_3}$, $\V\in\mathbb{R}^{n_2\times n_2\times n_3}$ are orthogonal, and $\SSS\in\mathbb{R}^{n_1\times n_2\times n_3}$ is an f-diagonal tensor. 
\end{thm}

\begin{figure}[!t]
	\centering
	\includegraphics[width=0.3\textwidth]{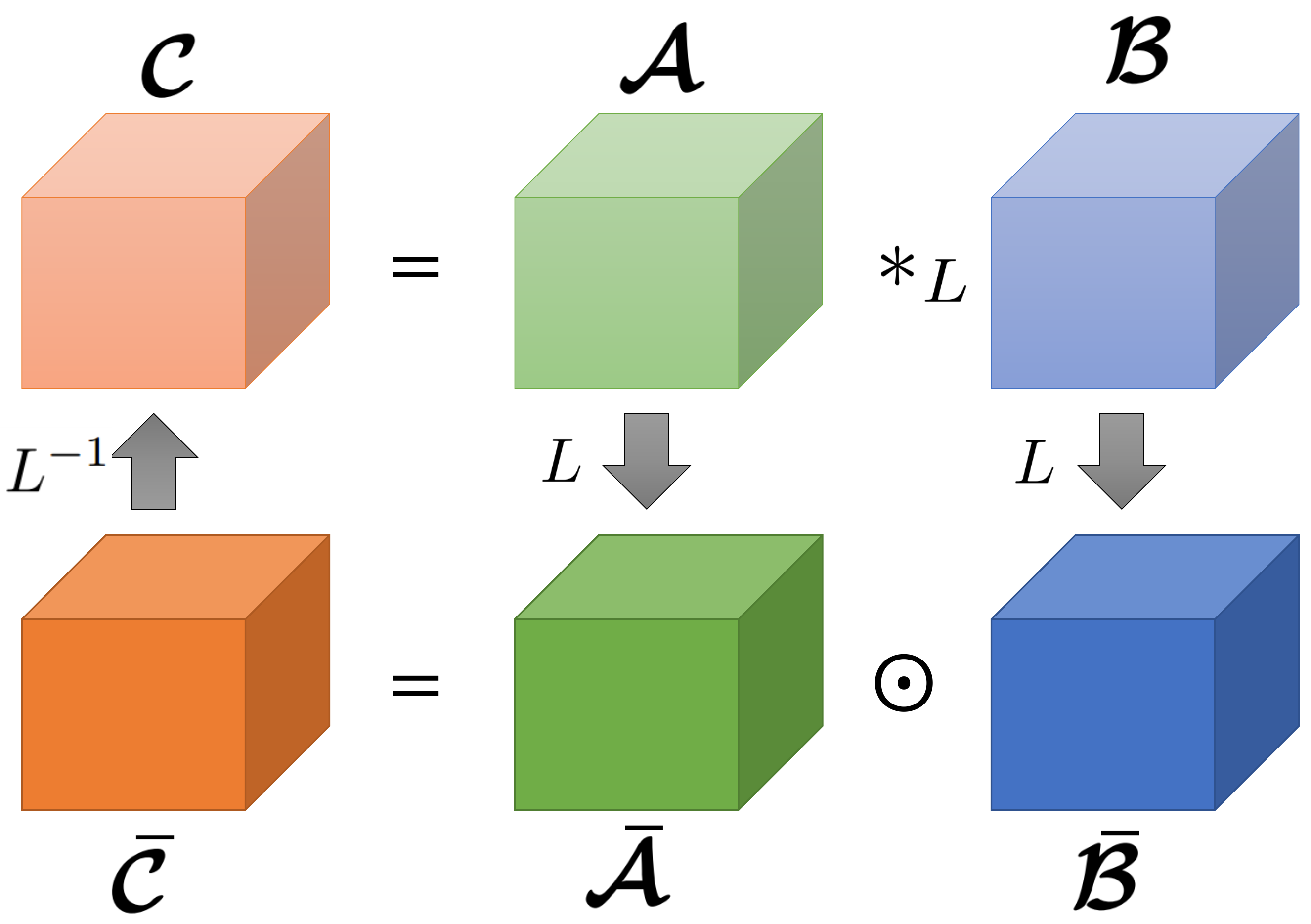}
	\caption{\small An illustration of the t-product  under linear transform $L$.}
	\label{fig_transtproduct}
\end{figure}

Figure \ref{fig_tsvd}  illustrates the t-SVD factorization. Also, t-SVD can be computed by performing matrix SVD in the transform domain. 
For any invertible linear transform $L$, we have $L(0)= L^{-1}(0) = 0$. So both $\SSS$ and $\Sbar$ are f-diagonal tensors. We can use their number of nonzero tubes to define the tensor tubal rank as  \cite{lu2018tensor}.

\begin{figure}[!t]
	\centering
	\includegraphics[width=0.5\textwidth]{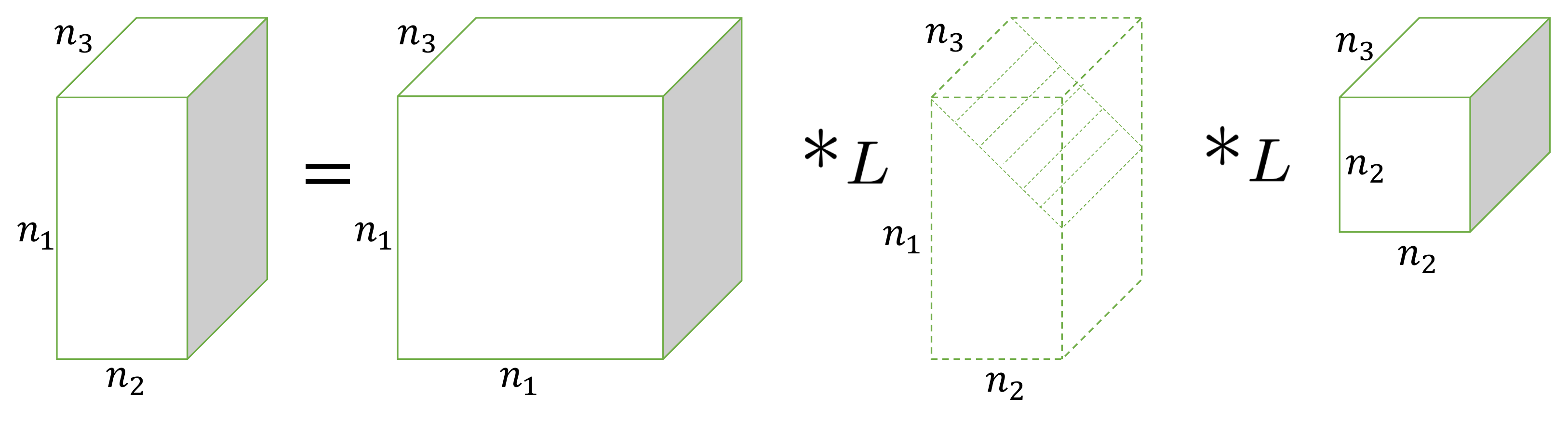}
	\caption{\small An illustration of the t-SVD under linear transform $L$ of a $n_1\times n_2\times n_3$ sized tensor.}
	\label{fig_tsvd}
\end{figure}

\begin{algorithm}[!t]
	\caption{T-SVD under linear transform $L$}
	\textbf{Input:} $\A\in\Rn$ and invertible linear transform $L$.\\
	\textbf{Output:} T-SVD components $\U$, $\SSS$ and $\V$ of $\A$.
	\begin{enumerate}[1.]
		\item Compute $\Abar=L(\A)$.
		\item Compute each frontal slice of $\Ubar$, $\Sbar$ and $\Vbar$ from $\Abar$ by
		
		\textbf{for} $i=1,\cdots,  n_3$ \textbf{do}
		
		\hspace*{0.4cm} $[\Umbar^{(i)},\Smbar^{(i)},\Vmbar^{(i)}] = \text{SVD}(\Ambar^{(i)})$;
		
		\textbf{end for}		
		\item Compute $\U=L^{-1}(\Ubar)$, $\SSS=L^{-1}(\Sbar)$, and $\V=L^{-1}(\Vbar)$.
	\end{enumerate}
	\label{alg_tsvd}	
\end{algorithm}

\begin{defn} \textbf{(Tensor tubal rank)}
	Let	$L$ be any invertible linear transform in (\ref{defL}), and $\A\in\Rn$. The tensor tubal rank of $\A$ under $L$, denoted as $\rankt(\A)$, is defined as the number of nonzero tubes of $\SSS$, where $\SSS$ is from the t-SVD of $\A=\U*_L\SSS*_L\V^\top$. We can write
	\begin{align*}
	\rankt(\A) = & \#\{i,\SSS(i,i,:)\neq {0}\} =  \#\{i,\Sbar(i,i,:)\neq {0}\}.
	\end{align*} 
\end{defn}

For $\A\in\Rn$ with tubal rank $r$, we also have the skinny t-SVD, {i.e.}, $\A=\U*_L\SSS*_L\V^\top$, where $\U\in\mathbb{R}^{n_1\times r\times n_3}$, $\SSS\in\mathbb{R}^{r\times r\times n_3}$, and $\V\in\mathbb{R}^{n_2\times r\times n_3}$, in which $\U^\top*_L\U=\I$ and $\V^\top*_L\V=\I$. We use the skinny t-SVD throughout this paper. The tensor tubal rank is nonconvex. In Section \ref{sec_rpca}, we will study the low tubal rank and sparse tensor decomposition problem by convex surrogate function minimization. At the following, we show how to define the convex tensor nuclear norm induced by the t-product under $L$. We can first define the tensor spectral norm as in \cite{lu2018tensor}.
\begin{defn}
	\textbf{(Tensor spectral norm)} Let $L$ be any invertible linear transform in (\ref{defL}), and $\A\in\Rn$. The tensor spectral norm of $\A$ under $L$ is defined as $\norm{\A} := \norm{\Ambar}$.
\end{defn}

The tensor nuclear norm can be defined as the dual norm of the tensor spectral norm. To this end, we need the following assumption on $L$ given in (\ref{defL}), i.e.,
\begin{align}\label{proL}
\Lm^\top \Lm = \Lm\Lm^\top = \ell  \Im_{n_3},
\end{align}
where $\ell>0$ is a constant. Using (\ref{proL}), we have the following important properties,
\begin{align}\label{eq_proinproduct}
\inproduct{\A}{\B}=\frac{1}{\ell}\inproduct{\Ambar}{\Bmbar},
\end{align}
\begin{align}\label{eq_proFnormNuclear}
\norm{\A}_F=\frac{1}{\sqrt{\ell}}\norm{\Ambar}_F.
\end{align}
For any $\B\in\Rn$ and $\Bmtilde\in\mathbb{R}^{n_1n_3\times n_2n_3}$, we have
\begin{align}
\norm{\A}_* 
:= & \sup_{\norm{\B}\leq1} \langle{\A},{\B}\rangle \label{eq_drivenuclearnorm0}\\
=&\sup_{\norm{\Bmbar}\leq1}\frac{1}{\ell}\langle{\Ambar},{\Bmbar}\rangle\label{eq_drivenuclearnorm1}\\
\leq &\frac{1}{\ell} \sup_{\norm{\Bmtilde}\leq1} \langle{\Ambar},{\Bmtilde}\rangle\label{eq_drivenuclearnorm2} \\
= &\frac{1}{\ell}\norm{\Ambar}_*,\label{eq_drivenuclearnorm3}
\end{align}
where (\ref{eq_drivenuclearnorm1}) uses (\ref{eq_proinproduct}), (\ref{eq_drivenuclearnorm2}) is due to the fact that $\Bmbar$ is a block diagonal matrix in $\mathbb{R}^{n_1n_3\times n_2n_3}$ while $\Bmtilde$ is an arbitrary matrix in $\mathbb{R}^{n_1n_3\times n_2n_3}$, and (\ref{eq_drivenuclearnorm3}) uses the fact that the matrix nuclear norm is the dual norm of the matrix spectral norm. On the other hand, let $\A=\U*_L\SSS*_L\V^\top$ be the t-SVD of $\A$ and $\B=\U*_L\V^\top$. We have
\begin{align}
\norm{\A}_* = & \sup_{\norm{\B}\leq1} \langle{\A},{\B}\rangle \notag\\
\geq & \langle{\U*_L\SSS*_L\V^\top},{\U*_L\V^\top}\rangle \notag\\
= & \langle{\U^\top*_L\U*_L\SSS},{\V^\top*_L\V}\rangle \notag\\
= & \langle{\SSS},\I\rangle \notag\\
= & \frac{1}{\ell} \langle \Sbar ,\Imbar \rangle\notag \\
= & \frac{1}{\ell} \Tr(\Smbar) \notag\\
= & \frac{1}{\ell} \norm{\Ambar}_*.    \label{defsipro}
\end{align}
Combining (\ref{eq_drivenuclearnorm0})-(\ref{eq_drivenuclearnorm3}) and (\ref{defsipro}), we then have the following definition of tensor nuclear norm.
\begin{defn}\label{defntnnours}
	\textbf{(Tensor nuclear norm)} Let $L$ be any invertible linear transform in (\ref{defL}) and it satisfies (\ref{proL}), and $\A=\U*_L\SSS*_L\V^\top$ be the t-SVD of $\A$. The tensor nuclear norm of $\A$ under $L$ is defined as $\norm{\A}_*:= \langle{\SSS},\I\rangle=\frac{1}{\ell} \norm{\Ambar}_*$.
\end{defn}
If we define the tensor  average rank as $\ranka(\A) = \frac{1}{\ell} \rankm(\Ambar)$, then it is easy to prove that the above tensor nuclear norm is the convex envelope of the tensor average rank within the domain $\{\A | \norm{\A} \leq 1 \}$ \cite{lu2018tensor}.



\section{TRPCA under Linear Transform and the Exact Recovery Guarantee}\label{sec_rpca}

With the new tensor rank and TNN definition in the previous section, now we consider the TRPCA problem and the convex model using TNN in (\ref{trpca}). Assume that we are given $\X = \LL_0 + \SSS_0$, where $\LL_0$ is of low tubal rank and $\SSS_0$ is sparse. The goal is to recover both $\LL_0$ and $\SSS_0$ from $\X$. We will show that both components can be exactly recovered by solving the convex program (\ref{trpca}) under certain incoherence conditions. In this section, we present the exact recovery guarantee of TRPCA in theory. We will also give the  detail for solving the convex program (\ref{trpca}).

\subsection{Tensor Incoherence Conditions}

As in the recovery problems of RPCA \cite{RPCA} and TRPCA~\cite{lu2016tensorrpca}, some incoherence conditions are required to avoid some pathological situations that the recovery is impossible. We need to assume that the low-rank component $\LL_0$ is not sparse. To this end, we assume $\LL_0$ to satisfy some incoherence conditions. Another identifiability issue arises if the sparse tensor has low tubal rank. This can be avoided by assuming that the support of $\SSS_0$ is uniformly distributed. We need the following tensor basis concept for defining the tensor incoherence conditions.

\begin{defn}\label{tensorbasis}
	\textbf{(Standard tensor basis)} Let $L$ be any invertible linear transform in (\ref{defL}) and it satisfies (\ref{proL}). We denote $\mathring{\mathfrak{e}}_i$ as the tensor \textbf{column basis}, which is a tensor of size	$n\times 1\times n_3$ with the entries of the  $(i,1,:)$ tube of $L(\mathring{\mathfrak{e}}_i)$ equaling 1 and the rest equaling 0.  Naturally its transpose $\mathring{\mathfrak{e}}_i^\top$ is called \textbf{row basis}. The other standard tensor basis is called \textbf{tube basis} $\dot{\mathfrak{e}}_k$, which is a tensor of size $1\times 1\times n_3$ with the $(1,1,k)$-th entry of $L(\dot{\mathfrak{e}}_k)$ equaling 1 and the rest equaling
	0.
\end{defn}

\begin{defn}\label{defticon}\textbf{(Tensor Incoherence Conditions)} 
	Let $L$ be any invertible linear transform in (\ref{defL}) and it satisfies (\ref{proL}). For $\LL_0\in\Rn$, assume that  $\rankt(\LL_0)=r$ and it has the skinny t-SVD  $\LL_0=\U*_L\SSS*_L\V^\top$, where $\U\in\mathbb{R}^{n_1\times r\times n_3}$ and
	$\V\in\mathbb{R}^{n_2\times r\times n_3}$ satisfy $\U^\top*_L\U=\I$ and $\V^\top*_L\V=\I$, and 
	$\SSS\in\mathbb{R}^{r\times r\times n_3}$ is a  f-diagonal tensor. Then $\LL_0$ is said to satisfy the tensor incoherence conditions with parameter $\mu$ if
	\begin{align}
	\max_{i=1,\cdots,n_1}\max_{k=1,\cdots,n_3} \norm{\U^\top*_L\mathring{\mathfrak{e}}_i *_L L(\dot{\mathfrak{e}}_k)}_F\leq\sqrt{\frac{\mu r}{n_1}}, \label{tic1}\\
	\max_{j=1,\cdots,n_2}\max_{k=1,\cdots,n_3} \norm{\V^\top*_L\mathring{\mathfrak{e}}_j*_L L(\dot{\mathfrak{e}}_k)}_F\leq\sqrt{\frac{\mu r}{n_2}},\label{tic2}
	\end{align}
	and
	\begin{equation}
	\norm{\U*_L\V^\top}_\infty\leq \sqrt{\frac{\mu r}{n_1n_2\ell}}.\label{tic3}
	\end{equation}	
\end{defn}
The incoherence condition guarantees that for small values of $\mu$, the singular vectors are reasonably spread out, or not sparse \cite{RPCA,candes2009exact}.  

\begin{prop}\label{pro_tic}
With the same notations in Definition \ref{defticon}, if the following conditions hold,
\begin{align}
\max_{i=1,\cdots,n_1} \norm{\U^\top*_L\mathring{\mathfrak{e}}_i}_F\leq\sqrt{\frac{\mu r}{n_1\ell}}, \label{ticv2}\\
\max_{j=1,\cdots,n_2} \norm{\V^\top*_L\mathring{\mathfrak{e}}_j}_F\leq\sqrt{\frac{\mu r}{n_2\ell}},\label{tic2v2}
\end{align}
then (\ref{tic1})-(\ref{tic2}) hold.
\end{prop}
\begin{proof}
	By using property (\ref{eq_proFnormNuclear}),  we have
	\begin{align*}
	& \norm{\U^\top*_L\mathring{\mathfrak{e}}_i *_L L(\dot{\mathfrak{e}}_k)}_F \\
	= &\frac{1}{\sqrt{\ell}} \norm{L(\U^\top) \odot L(\mathring{\mathfrak{e}}_i) \odot L(L(\dot{\mathfrak{e}}_k))}_F \\
	\leq &\frac{1}{\sqrt{\ell}} \norm{L(\U^\top) \odot L(\mathring{\mathfrak{e}}_i)}_F \norm{L(L(\dot{\mathfrak{e}}_k))}_F \\
	= & \sqrt{\ell} \norm{\U^\top*_L\mathring{\mathfrak{e}}_i }_F,
	\end{align*}
	where we use  $\norm{L(L(\dot{\mathfrak{e}}_k))}_F=\sqrt{\ell}\norm{L(\dot{\mathfrak{e}}_k)}_F=\sqrt{\ell}$ due to property (\ref{proL}). Thus, if  (\ref{ticv2}) hold, then (\ref{tic1}) hold. Similarly, we have the same relationship between  (\ref{tic2v2}) and (\ref{tic2}).
\end{proof}


The tensor incoherence conditions are used in the proof of Theorem \ref{thm1}. Though similar formulations as  (\ref{ticv2})-(\ref{tic2v2}) are  widely used in RPCA \cite{RPCA} and TRPCA \cite{lu2018tensor}, Proposition \ref{pro_tic} shows that our new conditions (\ref{tic1})-(\ref{tic2})  are less  restrictive. And thus we use the new conditions to replace  (\ref{ticv2})-(\ref{tic2v2}) in our proofs.

\subsection{Main Results}
Define $n_{(1)}=\max(n_1,n_2)$ and $n_{(2)}=\min ({n_1,n_2})$. We show that  the convex program (\ref{trpca}) is able to perfectly recover the low-rank  and sparse components.

\begin{thm}\label{thm1}
	Let $L$ be any invertible linear transform in (\ref{defL}) and it satisfies (\ref{proL}). Suppose $\LL_0\in \mathbb{R}^{n\times n\times n_3}$ obeys (\ref{tic1})-(\ref{tic3}). Fix any  $\nss$  tensor $\M$ of signs. Suppose that the support set $\Omegat$ of $\SSS_0$ is uniformly distributed among all sets of cardinality $m$, and that $\sgn{[\SSS_0]_{ijk}}=[\M]_{ijk}$ for all $(i,j,k)\in\Omegat$. Then, there exist universal constants $c_1, c_2>0$ such that with probability at least $1-c_1(nn_3)^{-c_2}$ (over the choice of  support of $\SSS_0$), $\{\LL_0,\SSS_0\}$ is the unique minimizer to (\ref{trpca}) with $\lambda = 1/\sqrt{n\ell}$, provided that
	\begin{equation}
	\rankt(\LL_0)\leq \frac{ \rho_r n  }{\mu(\log(nn_3))^{2} } \text{ and } m\leq \rho_sn^2n_3,
	\end{equation}
	where $\rho_r$ and $\rho_s$ are positive constants. If $\LL_0\in\Rn$ has rectangular frontal slices, TRPCA with $\lambda = 1/\sqrt{\none \ell}$ succeeds with probability at least $1-c_1(\none n_3)^{-c_2}$, provided that $\rankt(\LL_0)\leq \frac{ \rho_r \ntwo }{\mu(\log(\none n_3))^{2} } \text{ and } m\leq \rho_sn_1n_2 n_3$.
\end{thm}
Theorem \ref{thm1} gives the exact recovery guarantee for convex model (\ref{trpca}) under certain conditions. It says that the incoherent $\LL_0$ can be recovered for $\rankt(\LL_0)$ on the order of $n/(\mu(\log n n_3)^2)$ and a number of nonzero entries in $\SSS_0$ on the order of $n^2n_3$. 
TRPCA is also parameter free.
If $n_3=1$, both our model (\ref{trpca}) and result in Theorem \ref{thm1} reduces to the matrix RPCA cases in \cite{RPCA}. Intuitively, the TNN definition and its based TRPCA model in \cite{lu2018tensor} fall into the special cases of ours when the discrete Fourier transform is used. We would like to emphasize some more differences from \cite{lu2018tensor}  as follows:
\begin{enumerate}[1.]
	\item It is obvious our TNN and its based TRPCA model under linear transforms are much more general than the ones in \cite{lu2018tensor} which uses the discrete Fourier transform. We only require the linear transform satisfies property (\ref{proL}). We then can use different linear transforms for different purposes, e.g., improving the learning performance for different data or tasks or improving the efficiency by using special linear transforms.
	\item Our tensor incoherence conditions, theoretical exact recovery guarantee in Theorem \ref{thm1} and its proof do not fall into the cases in \cite{lu2018tensor} even the discrete Fourier transform is used. The key difference is that we restrict the linear transform within the real domain $L: \mathbb{R}^{n_1\times n_2\times n_3}\rightarrow\mathbb{R}^{n_1\times n_2\times n_3}$ while the discrete Fourier transform is a mapping from the real domain to the complex domain  $L: \mathbb{R}^{n_1\times n_2\times n_3}\rightarrow\mathbb{C}^{n_1\times n_2\times n_3}$. Several special properties of the discrete Fourier transform are used in their proof of exact recovery guarantee, but these techniques are not applicable in our proofs. For example, the definitions of the standard tensor basis are different. This leads to different tensor incoherence conditions and thus different proofs. We discuss more in the proofs which are provided in the supplementary material.
\end{enumerate}

%

Problem (\ref{trpca}) can be solved by the standard Alternating Direction Method of Multipliers (ADMM) \cite{lu2018unified}. We only give the detail for solving the following key subproblem in ADMM, i.e., for any $\Y \in\Rn$,

\begin{equation}\label{potnn}
\min_{\X} \ \tau \norm{\X}_*+\frac{1}{2}\norm{\X - \Y}_F^2.
\end{equation}
Let $\Y= \U *_L \SSS *_L \V^\top$ be the tensor SVD of $\Y$. For any $\tau>0$, we define the Tensor Singular Value Thresholding (T-SVT) operator as follows
\begin{equation}\label{eqn_tsvt}
\mathcal{D}_{\tau}(\Y) = \U *_L \SSS_\tau *_L\V^\top,
\end{equation}
where $\SSS_\tau = L^{-1}(( L(\SSS) -\tau)_+)$. Here $t_+ = \max(t,0)$.  The T-SVT operator is the proximity operator associated with the proposed tensor nuclear norm. 
\begin{thm}\label{thmtsvt} 
	Let $L$ be any invertible linear transform in (\ref{defL}) and it satisfies (\ref{proL}).
	For any $\tau>0$ and $\Y \in\Rn$, the T-SVT operator (\ref{eqn_tsvt}) obeys 
	\begin{equation}\label{thmeqnsvt}
	\mathcal{D}_{\tau}(\Y) = \arg\min_{\X} \ \tau \norm{\X}_*+\frac{1}{2}\norm{\X - \Y}_F^2.
	\end{equation}
\end{thm}

The main cost of   ADMM for solving (\ref{trpca}) is to compute $\mathcal{D}_{\tau}(\Y)$ for solving (\ref{thmeqnsvt}). For any general linear transform $L$ in (\ref{defL}), it is easy to see that the per-iteration cost is $O(n_1n_2n_3^2 + n_{(1)} n_{(2)}^2 n_3)$. Such a cost is the same as that in TRPCA using DFT as the linear transform \cite{lu2016tensorrpca}.

\section{Experiments}\label{sec_exp}
We  present an empirical study of our method. The goal of this study is two-fold: a) establish that
the convex program (\ref{trpca}) indeed recovers the low-rank and sparse parts correctly, and thus verify our  result in Theorem \ref{thm1}; b) show that our tensor methods are superior to matrix based RPCA methods and other existing TRPCA methods in practice.

\subsection{Exact Recovery from Varying Fractions of Error}

We first verify the correct recovery guarantee of Theorem \ref{thm1} on randomly generated tensors. Since Theorem \ref{thm1} holds for any invertible linear transform $L$ with $\Lm$ in (\ref{defL}) satisfying (\ref{proL}). So we consider two cases of $\Lm$: (1) Discrete Cosine Transform (DCT) which is used in the original work of transforms based t-product \cite{kernfeld2015tensor}. We use the Matlab command \mcode{dct} to generate the DCT matrix $\Lm$.  (2) Random Orthogonal Matrix (ROM) generated by the method with codes available online\footnote{https://www.mathworks.com/matlabcentral/fileexchange/11783-randorthmat.}. In both cases,   (\ref{proL}) holds with $\ell=1$.
We use the same way for generating the random data as~\cite{lu2018tensor}. We simply consider the tensors of size $n\times n \times n$, with $n=$100, 200 and 300. We generate a tensor with tubal rank $r$ as a product $\LL_0=\PP*_L\Q^\top$, where $\PP$ and $\Q$ are ${n\times r\times n}$ tensors with entries independently sampled from $\mathcal{N}(0,1/n)$ distribution. The support set $\Omegat$ (with size $m$) of $\SSS_0$ is chosen uniformly at random. For all $(i,j,k)\in\Omegat$, let $[\SSS_0]_{ijk}=\M_{ijk}$, where $\M$ is a tensor with independent Bernoulli $\pm 1$ entries. For different size $n$, we set the tubal rank of $\LL_0$ as $0.1n$ and consider two cases of the sparsity $m=\norm{\SSS_0}_0=0.1n^3$ and $0.2n^3$. 

\renewcommand{\arraystretch}{1.1}
\begin{table}[]
	\scriptsize
	\centering
	\caption{\small Correct recovery for random problems of varying size. The Discrete
		Cosine Transform (DCT) is used as the invertible linear transform $L$.}
	\label{my-label}
	$r=\rankt(\LL_0)=0.1n$, $m=\norm{\SSS_0}_0=0.1n^3$
	\begin{tabular}{c|c|c|c|c|c|c}
		\hline
		$n$ & $r$ & $m$      & $\rankt(\Lhat)$ &  $\|\Shat\|_0$     &  $\frac{\norm{\Lhat-\LL_0}_F}{\norm{\LL_0}_F}$   &  $\frac{\norm{\Shat-\SSS_0}_F}{\norm{\SSS_0}_F}$                             \\ \hline\hline
		100       & 10   & $1\e{5}$  & 10    & 102,921    & $2.4\e{-6}$ & $8.7\e{-10}$\\ \hline
		200       & 20   & $8\e{5}$ & 20    & 833,088     & $6.8\e{-6}$ & $8.7\e{-10}$\\ \hline
		300       & 30   & $27\e{5}$ & 30    & 2,753,084  & $1.8\e{-5}$ & $1.4\e{-9}$ \\ \hline
	\end{tabular}
	\vspace{0.1cm}
	
	$r=\rankt(\LL_0)=0.1n$, $m=\norm{\SSS_0}_0=0.2n^3$
	\begin{tabular}{c|c|c|c|c|c|c}
		\hline
		$n$ & $r$ & $m$      & $\rankt(\Lhat)$ &  $\|\Shat\|_0$     &  $\frac{\norm{\Lhat-\LL_0}_F}{\norm{\LL_0}_F}$   &  $\frac{\norm{\Shat-\SSS_0}_F}{\norm{\SSS_0}_F}$                             \\ \hline\hline
		100       & 10   & $2\e{5}$  & 10    & 201,090    & $4.0\e{-6}$ & $2.1\e{-10}$ \\ \hline
		200       & 20   & $16\e{5}$ & 20    & 1,600,491 & $5.3\e{-6}$ & $9.8\e{-10}$ \\ \hline
		300       & 30   & $54\e{5}$ & 30    & 5,460,221 & $1.8\e{-5}$ & $1.8\e{-9}$ \\ \hline
	\end{tabular}\label{tab_recov_dct}
\end{table}

\renewcommand{\arraystretch}{1.1}
\begin{table}[]
	\scriptsize
	\centering
	\caption{\small Correct recovery for random problems of varying size. A Random Orthogonal Matrix (ROM) is used as the invertible linear transform $L$.}\vspace{-0.2cm}
	$r=\rankt(\LL_0)=0.1n$, $m=\norm{\SSS_0}_0=0.1n^3$
	\begin{tabular}{c|c|c|c|c|c|c}
		\hline
		$n$ & $r$ & $m$      & $\rankt(\Lhat)$ &  $\|\Shat\|_0$     &  $\frac{\norm{\Lhat-\LL_0}_F}{\norm{\LL_0}_F}$   &  $\frac{\norm{\Shat-\SSS_0}_F}{\norm{\SSS_0}_F}$                             \\ \hline\hline
		100       & 10   & $1\e{5}$  & 10    & 103,034    & $2.7\e{-6}$ & $9.6\e{-9}$\\ \hline
		200       & 20   & $8\e{5}$ & 20    & 833,601     & $7.0\e{-6}$ & $9.0\e{-10}$\\ \hline
		300       & 30   & $27\e{5}$ & 30    & 2,852,933  & $1.8\e{-5}$ & $1.3\e{-9}$ \\ \hline
	\end{tabular}
	\vspace{0.1cm}
	
	$r=\rankt(\LL_0)=0.1n$, $m=\norm{\SSS_0}_0=0.2n^3$
	\begin{tabular}{c|c|c|c|c|c|c}
		\hline
		$n$ & $r$ & $m$      & $\rankt(\Lhat)$ &  $\|\Shat\|_0$     &  $\frac{\norm{\Lhat-\LL_0}_F}{\norm{\LL_0}_F}$   &  $\frac{\norm{\Shat-\SSS_0}_F}{\norm{\SSS_0}_F}$                             \\ \hline\hline
		100       & 10   & $2\e{5}$  & 10    & 201,070    & $4.3\e{-6}$ & $2.3\e{-9}$ \\ \hline
		200       & 20   & $16\e{5}$ & 20    & 1,614,206 & $ 5.4\e{-6}$ & $9.9\e{-10}$ \\ \hline
		300       & 30   & $54\e{5}$ & 30    & 5,457,874 & $9.6\e{-6}$ & $9.5\e{-10}$ \\ \hline
	\end{tabular}\label{tab_recov_rom}
\end{table}

Table \ref{tab_recov_dct} and \ref{tab_recov_rom} report the recovery results which use the DCT and ROM as the linear transform $L$, respectively.  It can be seen that our convex program (\ref{trpca}) gives the correct tubal rank estimation of $\LL_0$ in all cases and also the relative errors ${\|{\hat{\LL}-\LL_0}\|_F}/{\norm{\LL_0}_F}$ are very small, all around $10^{-5} $. The sparsity estimation of $\SSS_0$ is not as exact as the rank estimation, but note that the relative errors ${\|{\Shat-\SSS_0}\|_F}/{\norm{\SSS_0}_F}$ are all very small, less than $10^{-8}$ (actually much smaller than the relative errors of the recovered low-rank component). These results well verify the correct recovery phenomenon as claimed in Theorem \ref{thm1} under different linear transforms.

\subsection{Phase Transition in Tubal Rank and Sparsity}

Theorem \ref{thm1} shows that the recovery is correct when the tubal rank of $\LL_0$ is relatively low and $\SSS_0$ is relatively sparse. Now we examine the recovery phenomenon with varying tubal rank of $\LL_0$ from varying sparsity of $\SSS_0$. {We consider the tensors of  size $\mathbb{R}^{n\times n\times n_3}$, where $n=100$ and $n_3=50$. We generate $\LL_0$ by the same way as the previous section. We consider two distributions of the support of $\SSS_0$. The first case is the Bernoulli model for the support of $\SSS_0$, with random signs: each entry of $\SSS_0$ takes on value 0 with probability $1-\rho$, and values $\pm1$ each with probability $\rho/2$. The second case chooses the support $\bm\Omega$ in accordance with the Bernoulli model, but this time sets $\SSS_0=\Pomega \text{sgn}(\LL_0)$.} We set $\frac{r}{n}=[0.01:0.01:0.5]$ and $\rho_s=[0.01:0.01:0.5]$. For each $(\frac{r}{n},\rho_s)$-pair, we simulate 10 test instances and declare a trial to be successful if the recovered $\hat{\LL}$ satisfies ${\|{\hat{\LL}-\LL_0}\|_F}/{\norm{\LL_0}_F}\leq 10^{-3}$. Figure \ref{fig_sparsityvsrank} and \ref{fig_sparsityvsrank_sign} respectively show the the fraction of correct recovery for each pair $(\frac{r}{n},\rho_s)$ for two settings of $\SSS_0$. The white region indicates the exact recovery while the black one indicates the failure.  The experiment shows that the recovery is correct when the tubal rank of $\LL_0$ is relatively low and the errors $\SSS_0$ is relatively sparse. More importantly,   the results show that the used linear transforms are not important, as long as the property (\ref{proL}) holds. The recovery performances are very similar even different linear transforms are used. This well verify our main result in Theorem \ref{thm1}. The recovery phenomenons are also similar to RPCA \cite{RPCA} and TRPCA~\cite{lu2018tensor}. But note that our results are much more general as we can use different linear transforms. 

\begin{figure}[!t]
	\centering
	\begin{subfigure}[b]{0.235\textwidth}
		\centering
		\includegraphics[width=\textwidth]{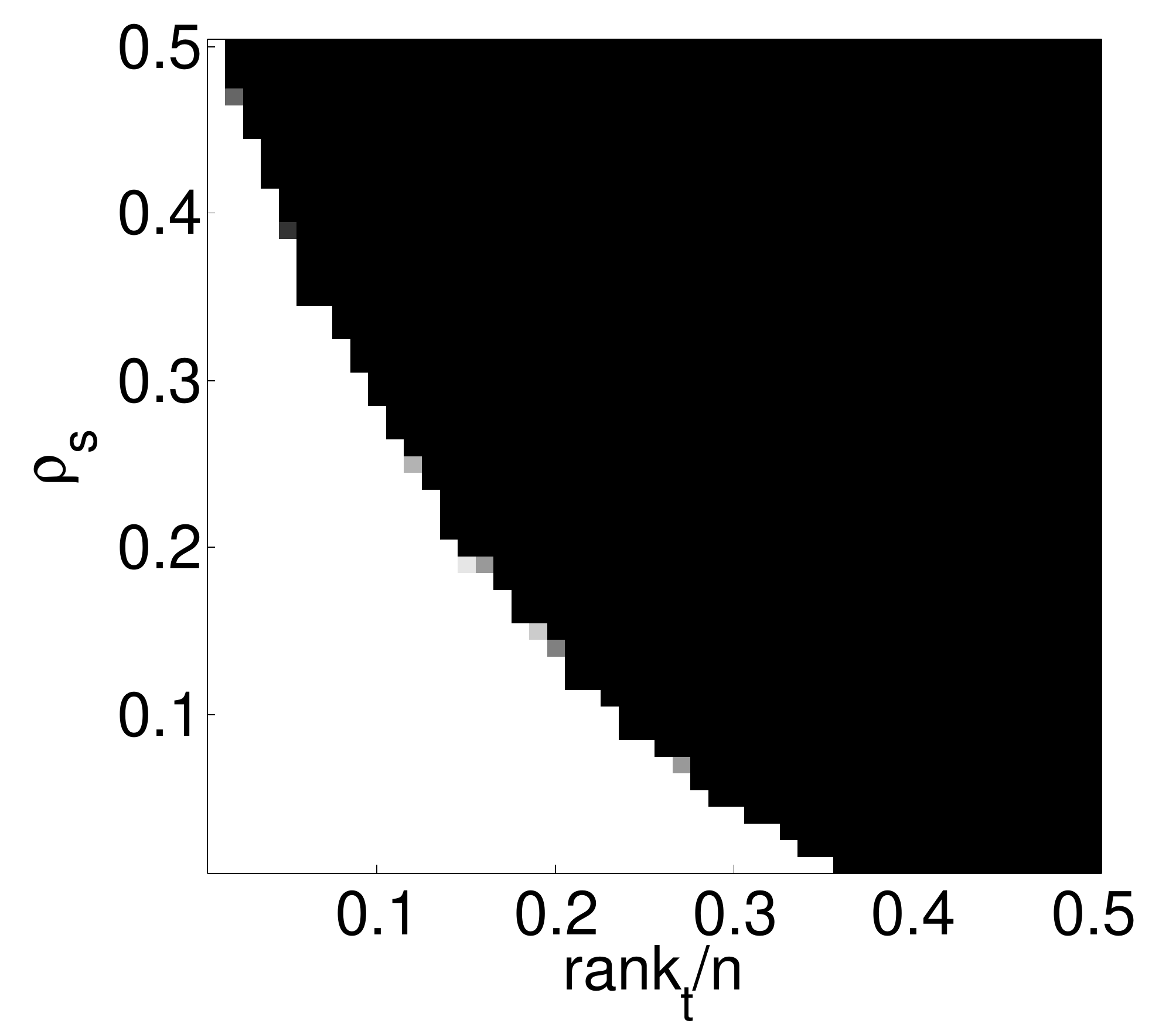}
		\caption{Random Signs, DCT}
	\end{subfigure}
	\begin{subfigure}[b]{0.235 \textwidth}
		\centering
		\includegraphics[width=\textwidth]{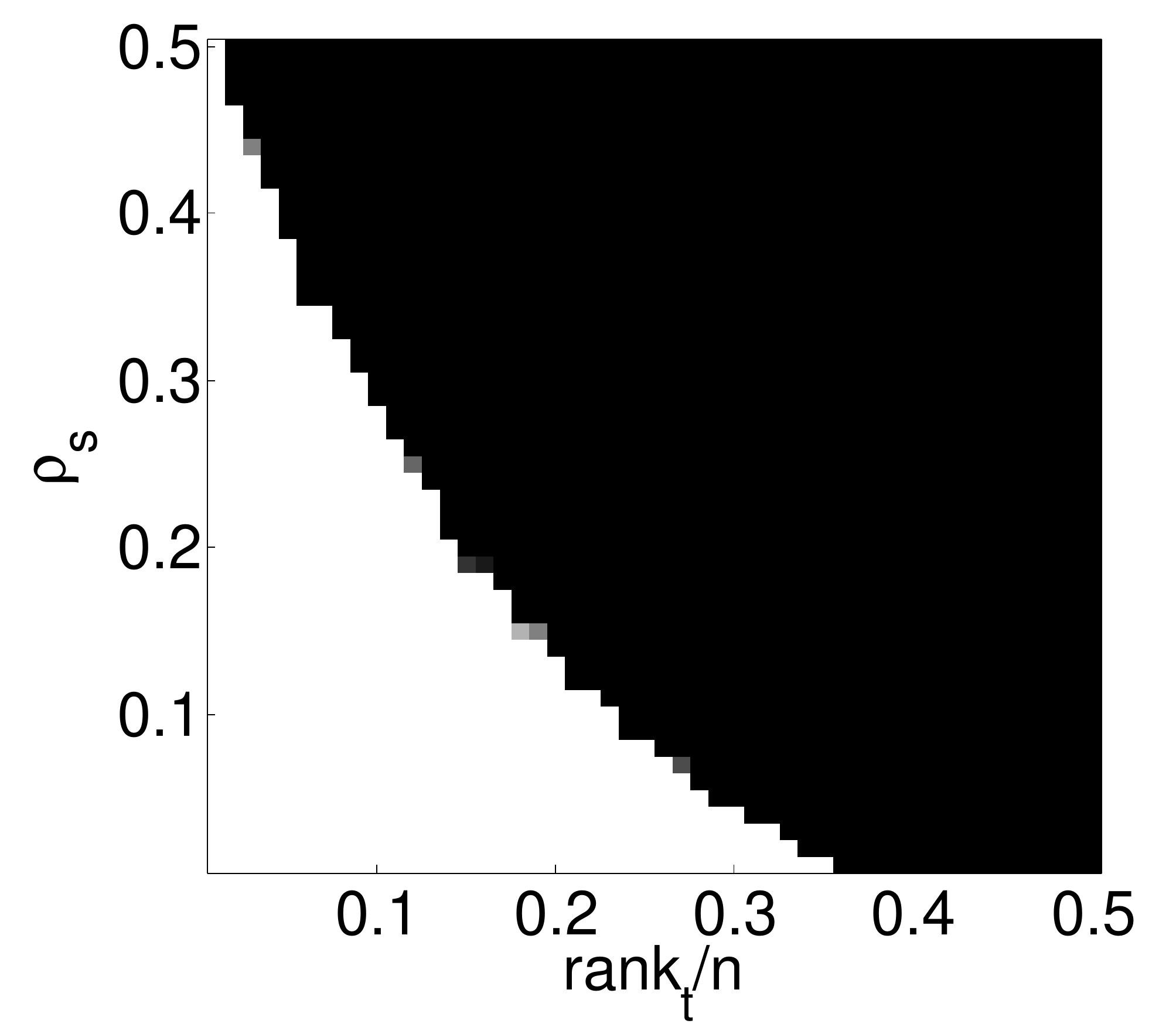}
		\caption{Random Signs, ROM}
	\end{subfigure}
	\caption{\small{Correct recovery for varying rank and sparsity. In this test, sgn($\SSS_0$) is random. Fraction of correct recoveries across 10 trials, as a function of $\rankt(\LL_0)$ (x-axis) and sparsity of $\SSS_0$ (y-axis). Left: DCT is used as the linear transform $L$. Right: ROM is used as the linear transform $L$.}}
	\label{fig_sparsityvsrank} 
\end{figure}

\begin{figure}[!t]
	\centering
	\begin{subfigure}[b]{0.235\textwidth}
		\centering
		\includegraphics[width=\textwidth]{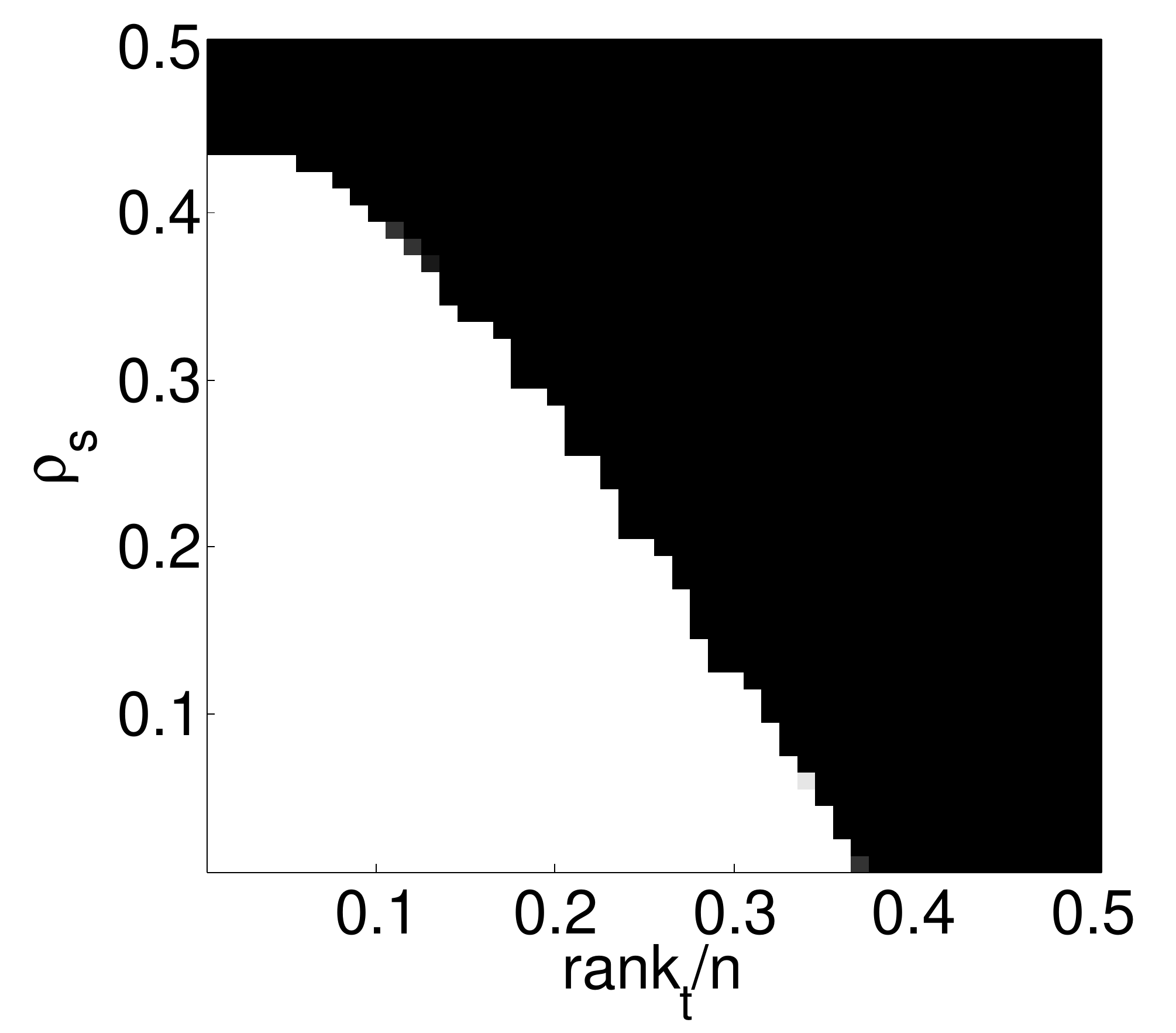}
		\caption{\small{Coherent Signs, DCT}}
	\end{subfigure}
	\begin{subfigure}[b]{0.235 \textwidth}
		\centering
		\includegraphics[width=\textwidth]{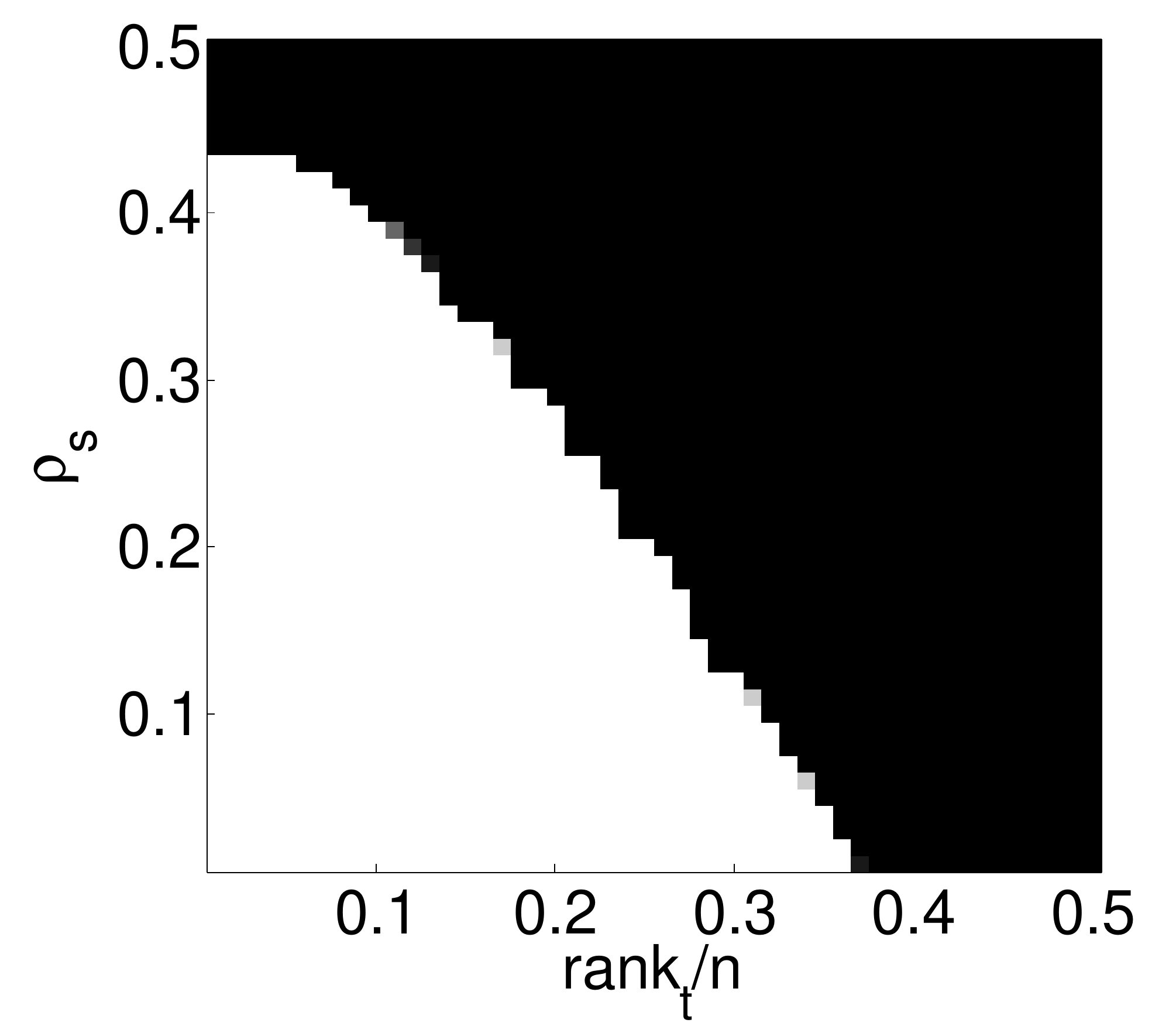}
		\caption{\small{Coherent Signs, ROM}}
	\end{subfigure}
	\caption{\small{Correct recovery for varying rank and sparsity. In this test, $\SSS_0 = \Pomega\text{sgn}(\LL_0)$. Fraction of correct recoveries across 10 trials, as a function of $\rankt(\LL_0)$ (x-axis) and sparsity of $\SSS_0$ (y-axis). Left: DCT is used as the linear transform $L$. Right: ROM is used as the linear transform $L$.}}
	\label{fig_sparsityvsrank_sign}
\end{figure}



\subsection{Applications to Image Recovery}

RPCA and tensor RPCA have been applied for image recovery \cite{liu2013tensor,lu2016tensorrpca,lu2018exact}. The main motivation is that the color image can be well approximated by low-rank matrices or tensors. In this section, we consider to apply our TRPCA model using DCT as the linear transform for image recovery, and  compare our method with state-of-the-art methods.


For any color image of size $n_1\times n_2$, we can formate it as a tensor $n_1\times n_2 \times n_3$, where $n_3=3$. The frontal slices correspond to the three channels of the color image\footnote{We observe that this way of tensor construction achieves the best performance in practice, though we do not show the results of other ways of tensor construction due to space limit.}. We randomly select 100 color images from the Berkeley segmentation dataset~\cite{martin2001database} for this test. 
We randomly set $10\%$ of pixels to random values in [0, 255], and the positions of the corrupted pixels are unknown. We compare our TRPCA model with RPCA \cite{RPCA}, SNN \cite{huang2014provable}  and TRPCA~\cite{lu2016tensorrpca}, which also own the theoretical recovery guarantee. For RPCA, we apply it on each channel separably and combine the results to obtain the recovered image.  For SNN \cite{huang2014provable} , we set $\bm{\lambda}=[15, 15, 1.5]$ as in \cite{lu2016tensorrpca}. TRPCA \cite{lu2016tensorrpca} is a special case of our new TRPCA using discrete Fourier transform. In this experiment, we use DCT and ROM as the transforms in our new TRPCA method. We denote our method using these two transforms as TRPCA-DCT and TRPCA-ROM, respectively. The Peak Signal-to-Noise Ratio (PSNR) value
\begin{equation*}\label{psnreq}
\text{PSNR}(\Lhat,\LL_0) = 10\log_{10}\left( \frac{\norm{\LL_0}_\infty^2}{\frac{1}{n_1n_2n_3} {\|\Lhat-\LL_0\|_F^2}} \right), 
\end{equation*}
is used to evaluate the recovery performance. The higher PSNR value indicates better recovery performance.

Figure \ref{fig_imgres} shows the comparison of the PSNR values and running time of all the compared methods. Some examples are given in Figure \ref{fig_imageinpr}. It can be seen that in most case, our new TRPCA-DCT achieves the best performance. 
This implies that the discrete Fourier transform used in \cite{lu2016tensorrpca} may not be the best in this task, though the reason is not very clear now. However, if
the random orthogonal matrix (ROM) is used as the linear transform, the results are generally bad. This is reasonable and it implies that the choice of the linear transform $L$ is crucial in practice, though the best linear transform is currently not clear. Figure \ref{fig_imgres} (b) shows that our method is as efficient as RPCA and TRPCA in \cite{lu2016tensorrpca}.

\begin{figure*}[!t]
	\centering
	\begin{subfigure}[b]{0.95\textwidth}
		\centering
		\includegraphics[width=\textwidth]{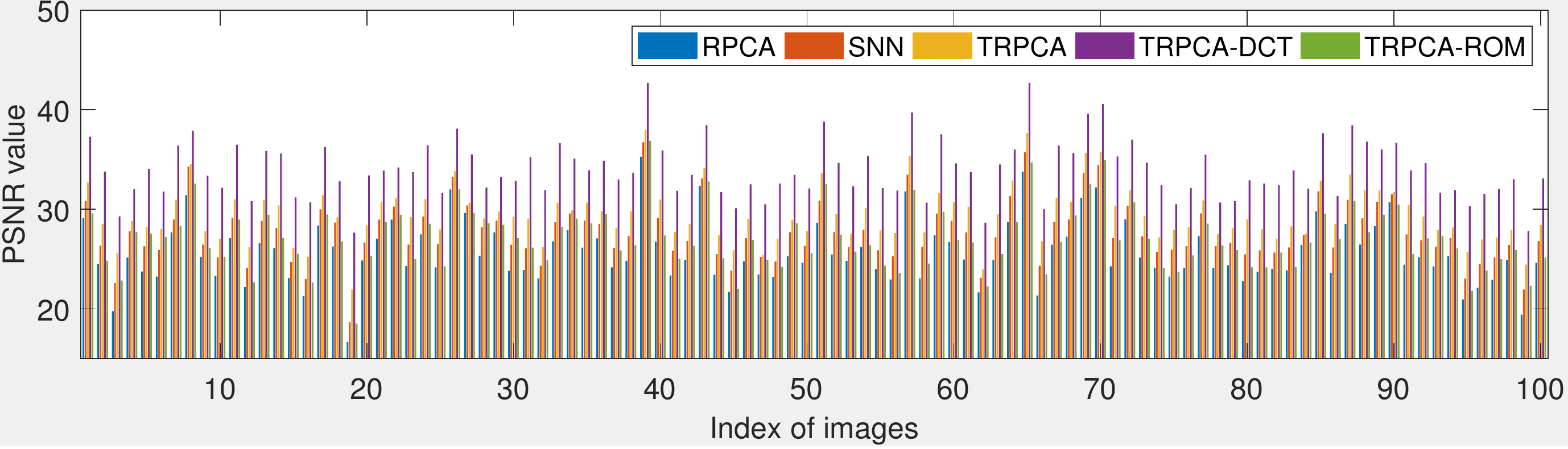}
	\end{subfigure}
	\begin{subfigure}[b]{0.95\textwidth}
		\centering
		\includegraphics[width=\textwidth]{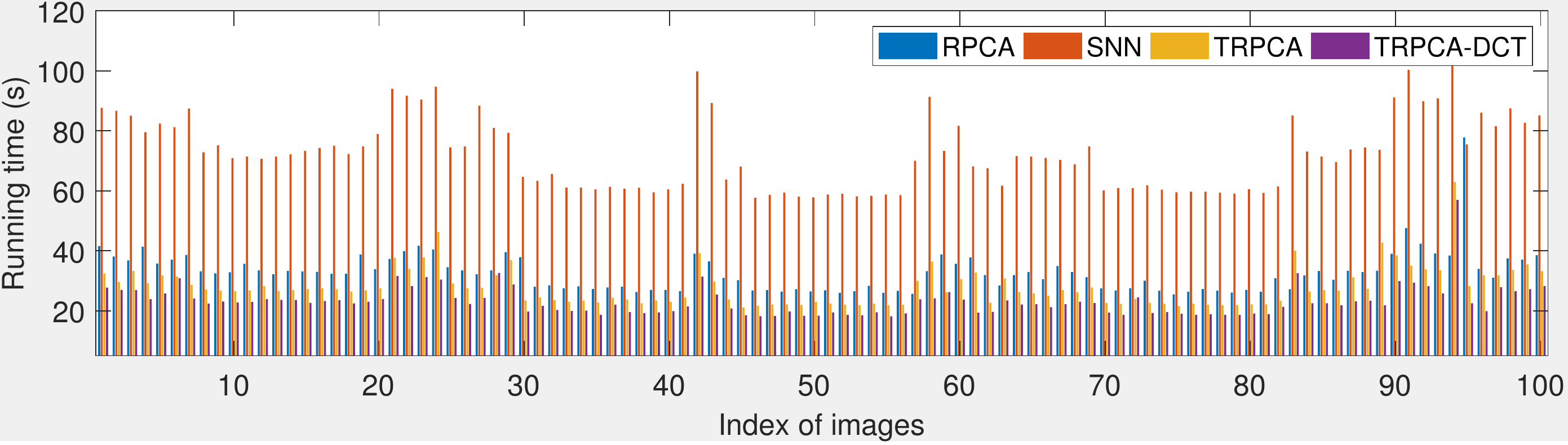}
	\end{subfigure}
	\caption{\small{ Comparison of the PSNR values (up) and running time (bottom) of RPCA, SNN, TRPCA, our TRPCA-DCT and TRPCA-ROM for image denoising on 50 images. \textbf{Best viewed in $\times2$ sized color pdf file.} }}
	\label{fig_imgres}
\end{figure*}

\begin{figure*}[!t]
	\centering	
	\begin{subfigure}[b]{0.16\textwidth}
		\centering
		\includegraphics[width=\textwidth]{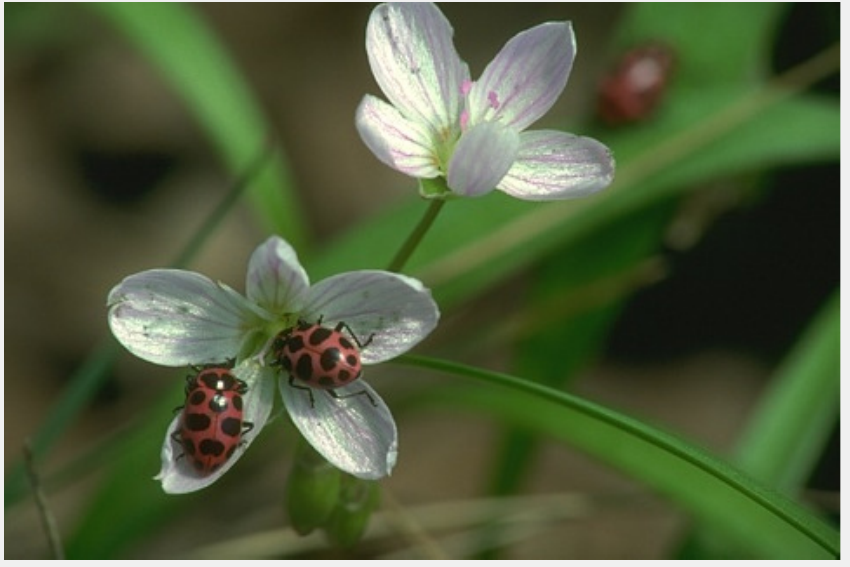}
	\end{subfigure}
	\begin{subfigure}[b]{0.16\textwidth}
		\centering
		\includegraphics[width=\textwidth]{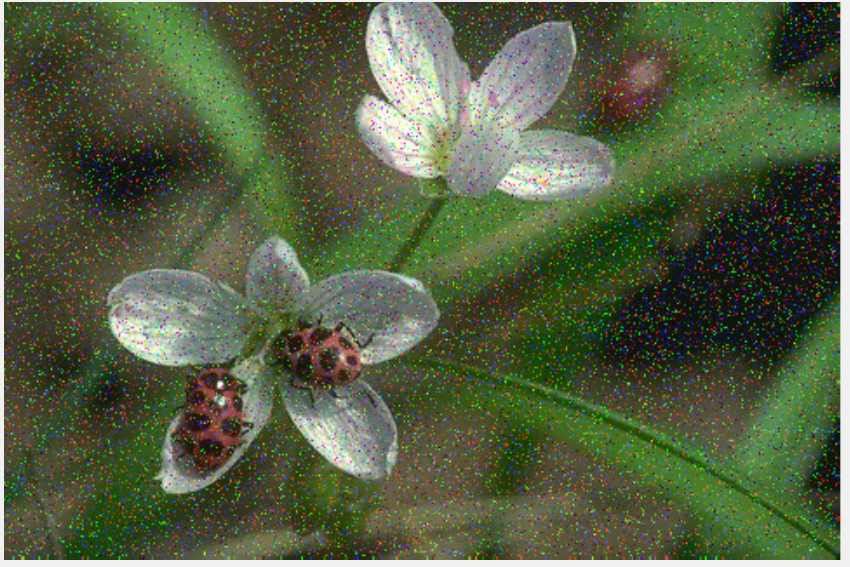}
	\end{subfigure}
	\begin{subfigure}[b]{0.16\textwidth}
		\centering
		\includegraphics[width=\textwidth]{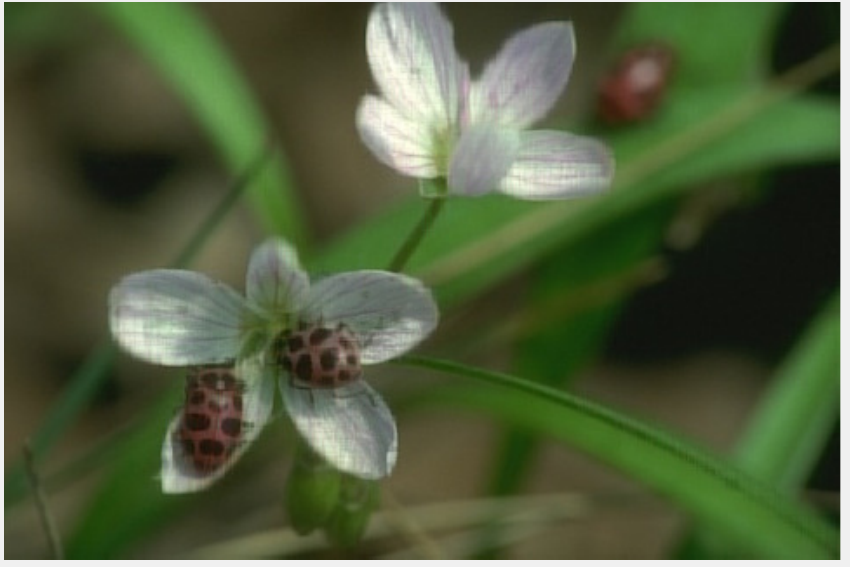}
	\end{subfigure}
	\begin{subfigure}[b]{0.16\textwidth}
		\centering
		\includegraphics[width=\textwidth]{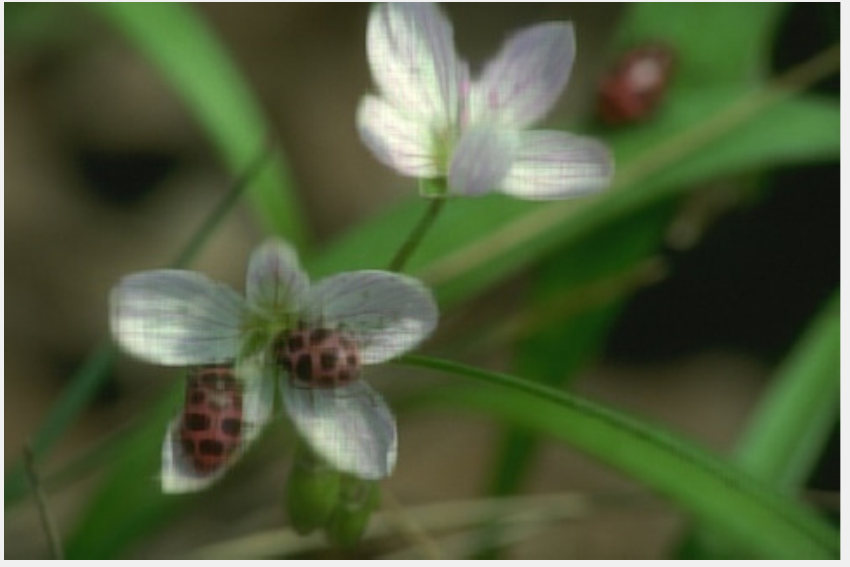}
	\end{subfigure}
	\begin{subfigure}[b]{0.16\textwidth}
		\centering
		\includegraphics[width=\textwidth]{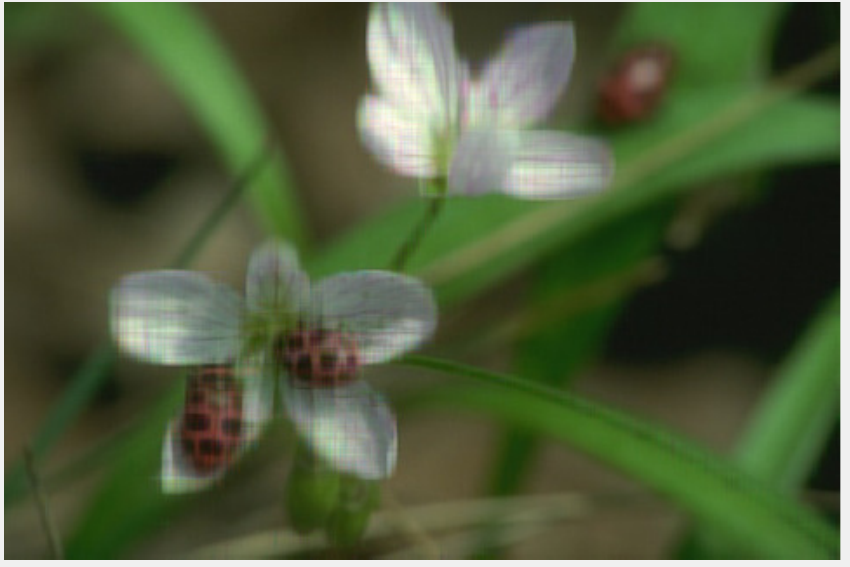}
	\end{subfigure}	
	\begin{subfigure}[b]{0.16\textwidth}
		\centering
		\includegraphics[width=\textwidth]{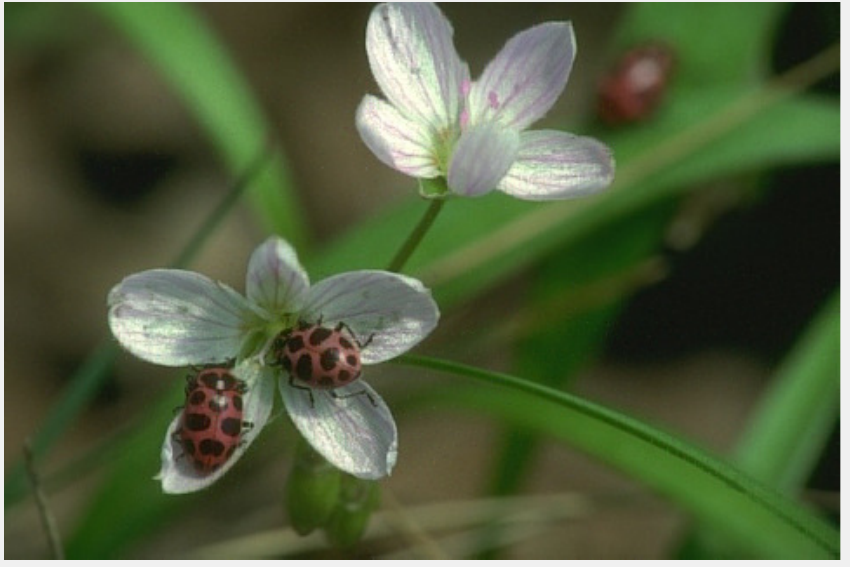}
	\end{subfigure}		
	\\\vspace{0.1em} 
	
	\begin{subfigure}[b]{0.16\textwidth}
		\centering
		\includegraphics[width=\textwidth]{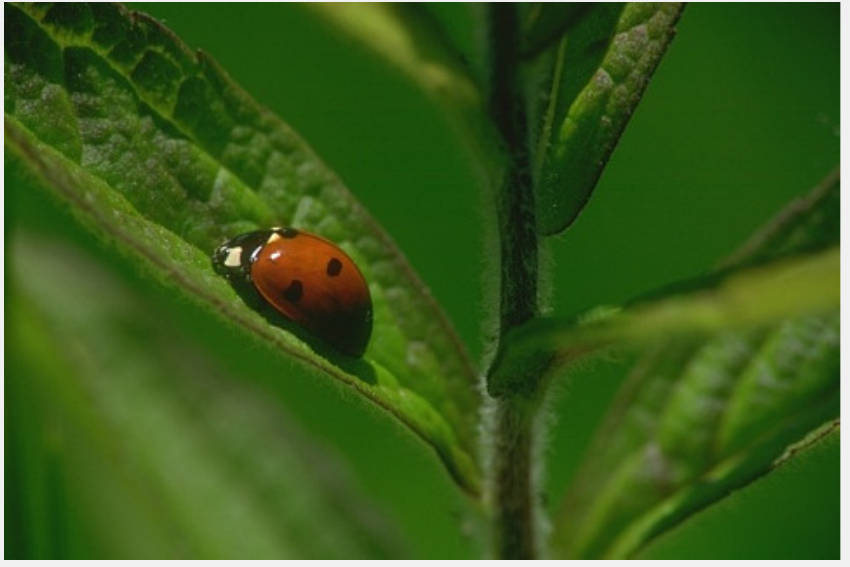}
	\end{subfigure}
	\begin{subfigure}[b]{0.16\textwidth}
		\centering
		\includegraphics[width=\textwidth]{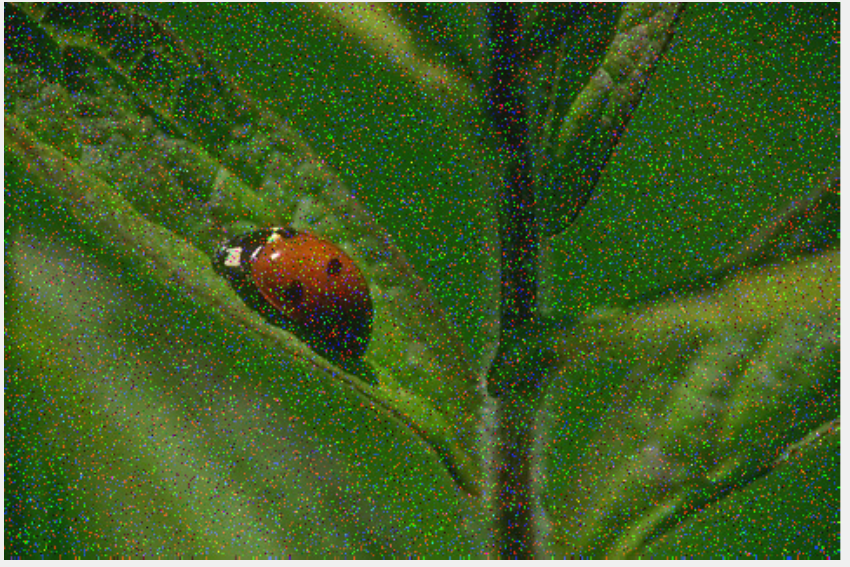}
	\end{subfigure}
	\begin{subfigure}[b]{0.16\textwidth}
		\centering
		\includegraphics[width=\textwidth]{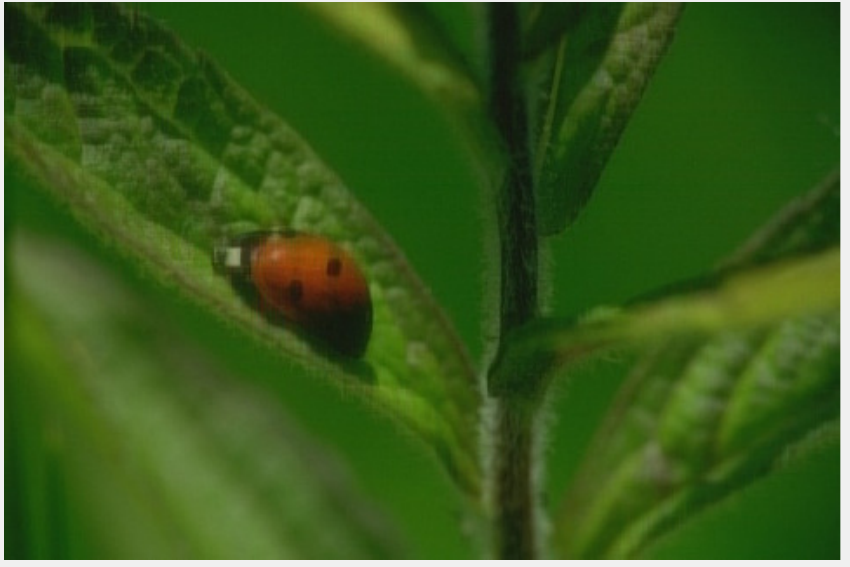}
	\end{subfigure}
	\begin{subfigure}[b]{0.16\textwidth}
		\centering
		\includegraphics[width=\textwidth]{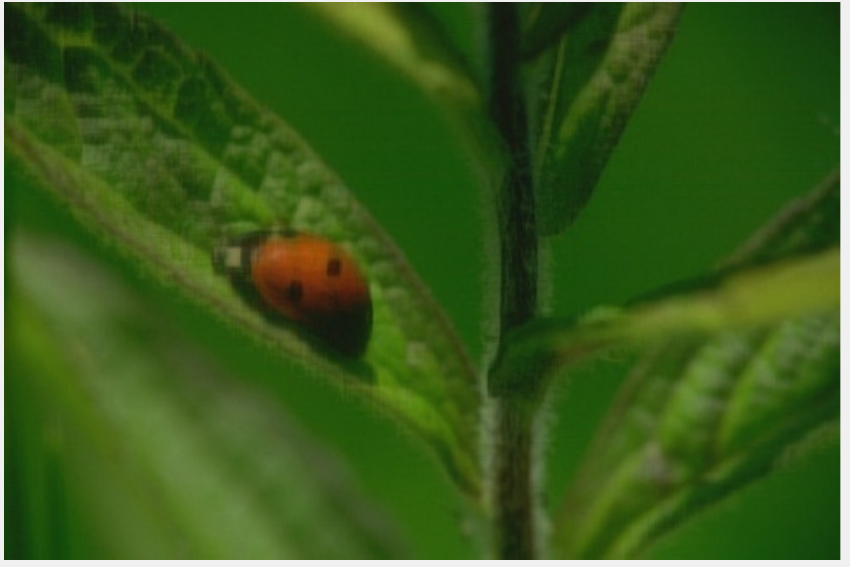}
	\end{subfigure}
	\begin{subfigure}[b]{0.16\textwidth}
		\centering
		\includegraphics[width=\textwidth]{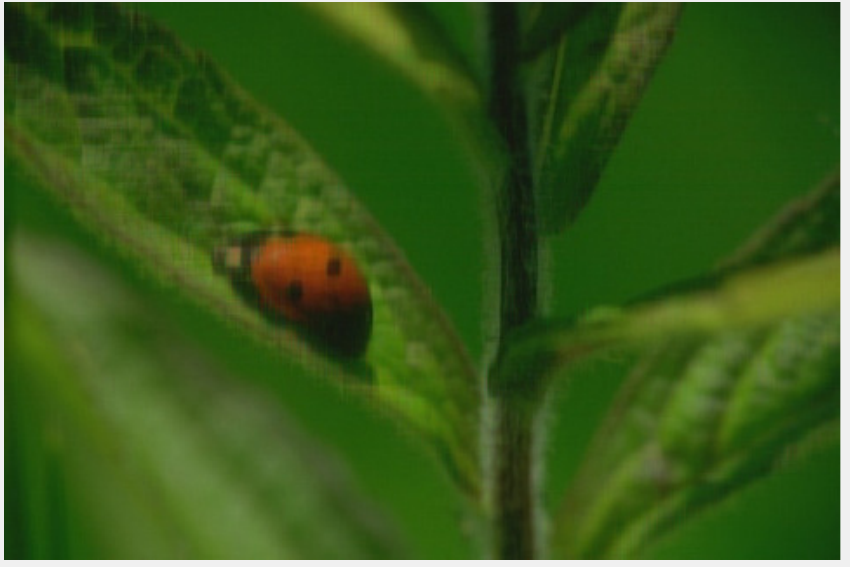}
	\end{subfigure}	
	\begin{subfigure}[b]{0.16\textwidth}
		\centering
		\includegraphics[width=\textwidth]{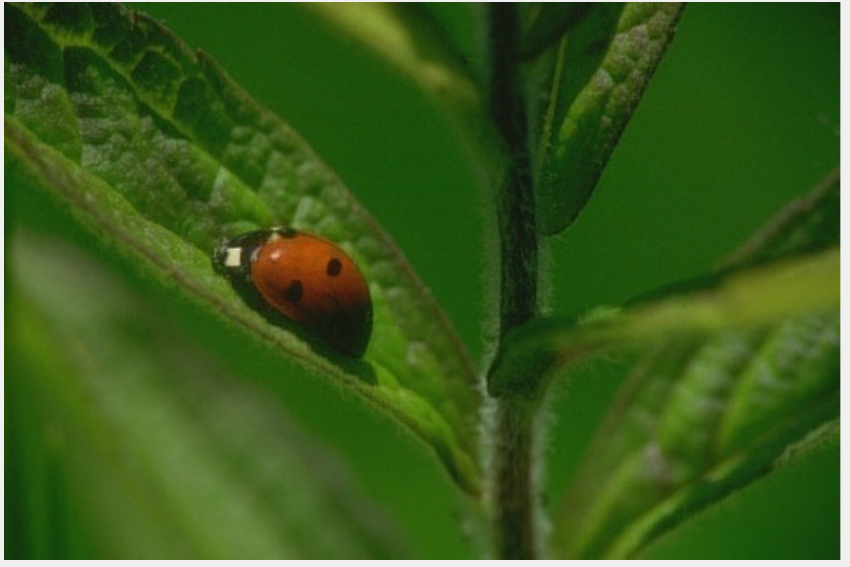}
	\end{subfigure}		
	\\\vspace{0.15em} 
	
	\begin{subfigure}[b]{0.16\textwidth}
		\centering
		\includegraphics[width=\textwidth]{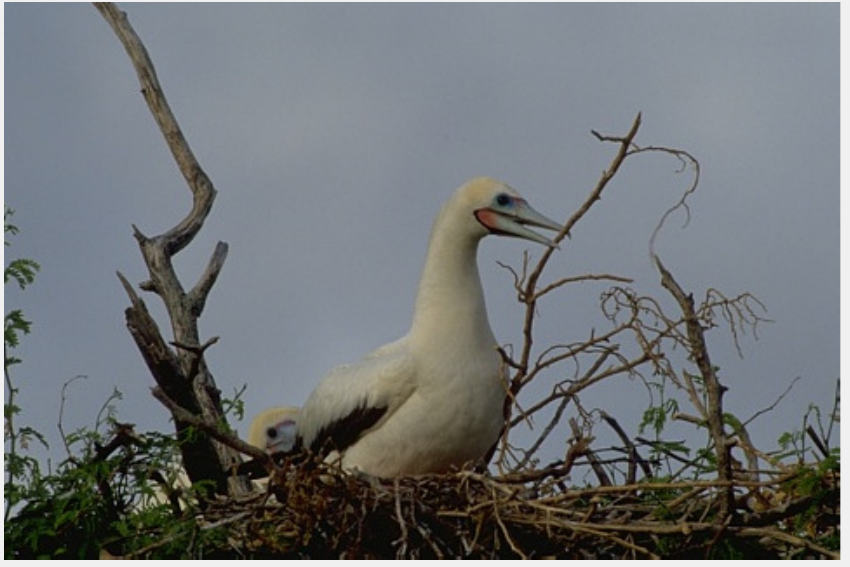}
	\end{subfigure}
	\begin{subfigure}[b]{0.16\textwidth}
		\centering
		\includegraphics[width=\textwidth]{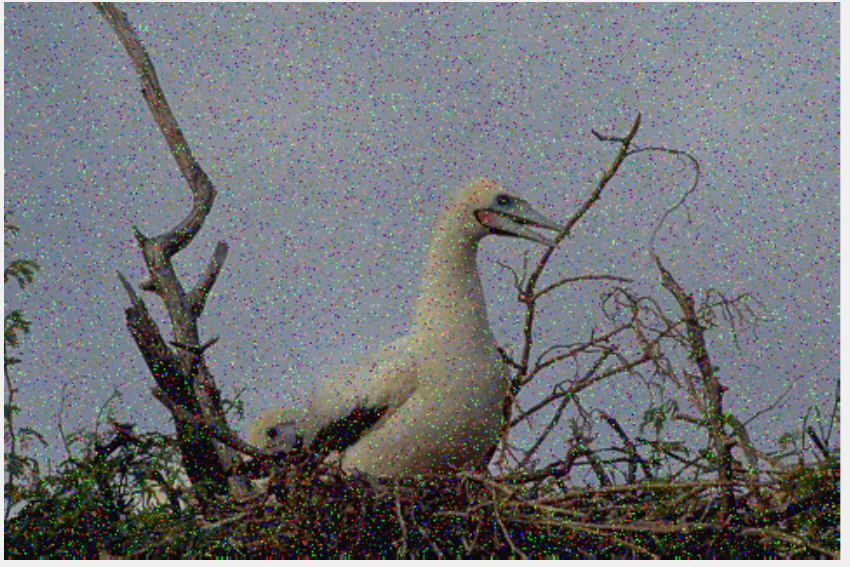}
	\end{subfigure}
	\begin{subfigure}[b]{0.16\textwidth}
		\centering
		\includegraphics[width=\textwidth]{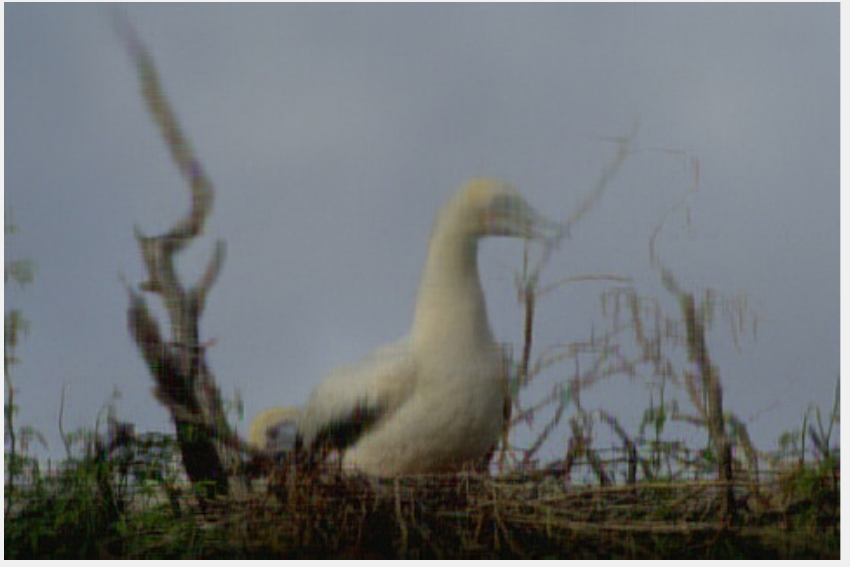}
	\end{subfigure}
	\begin{subfigure}[b]{0.16\textwidth}
		\centering
		\includegraphics[width=\textwidth]{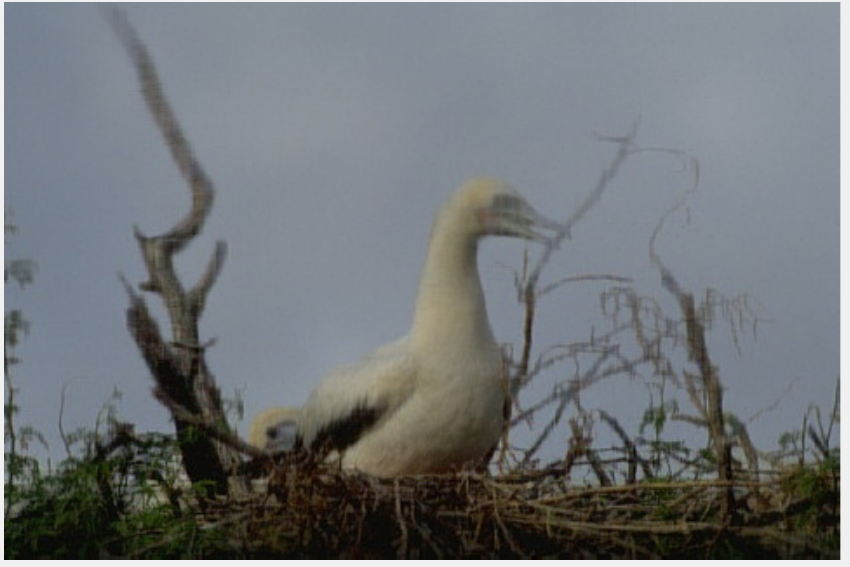}
	\end{subfigure}
	\begin{subfigure}[b]{0.16\textwidth}
		\centering
		\includegraphics[width=\textwidth]{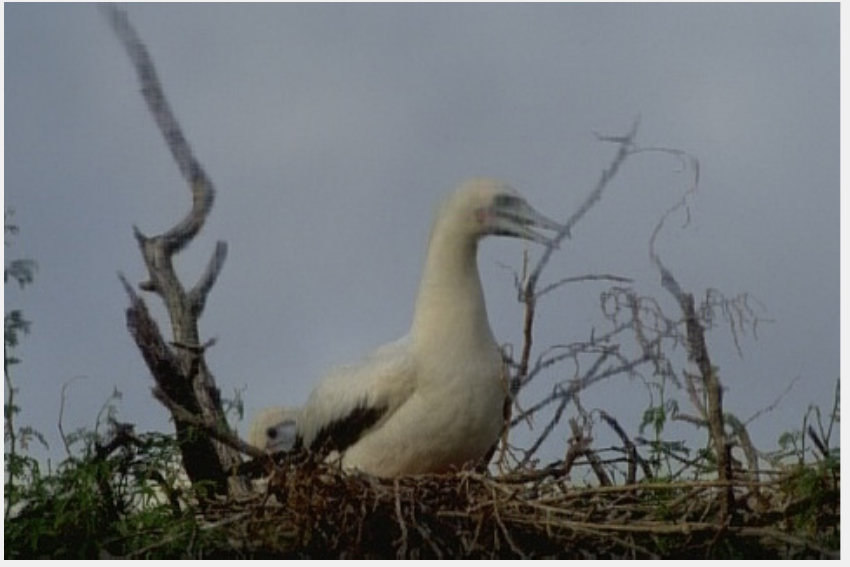}
	\end{subfigure}	
	\begin{subfigure}[b]{0.16\textwidth}
		\centering
		\includegraphics[width=\textwidth]{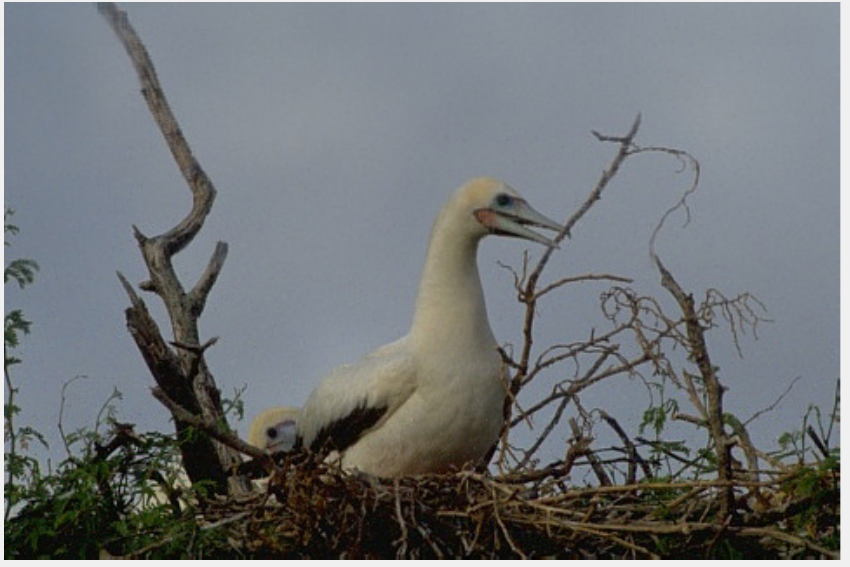}
	\end{subfigure}		
	\\\vspace{0.15em} 
	
	\begin{subfigure}[b]{0.16\textwidth}
		\centering
		\includegraphics[width=\textwidth]{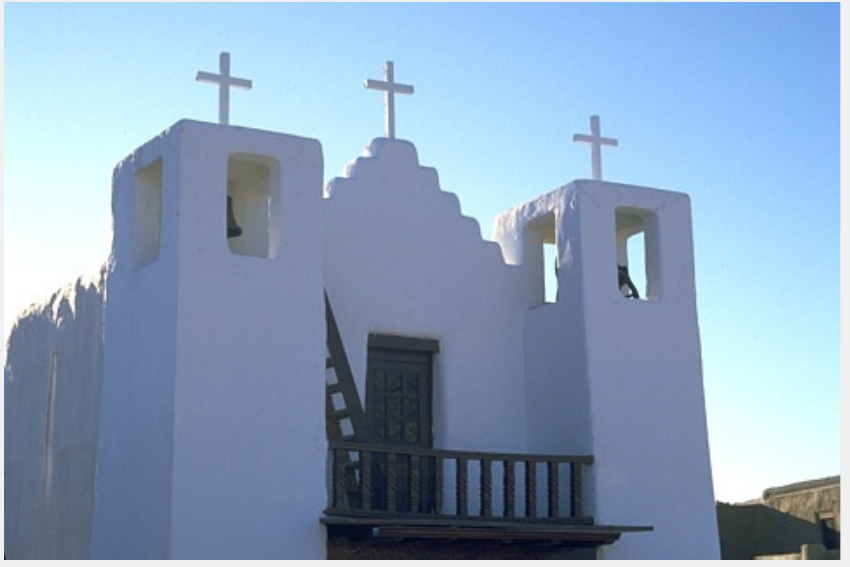}
	\end{subfigure}
	\begin{subfigure}[b]{0.16\textwidth}
		\centering
		\includegraphics[width=\textwidth]{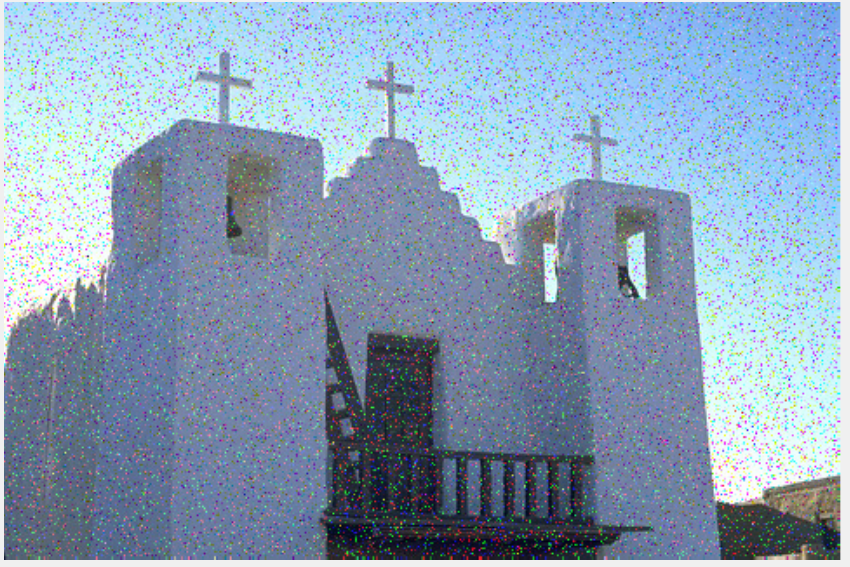}
	\end{subfigure}
	\begin{subfigure}[b]{0.16\textwidth}
		\centering
		\includegraphics[width=\textwidth]{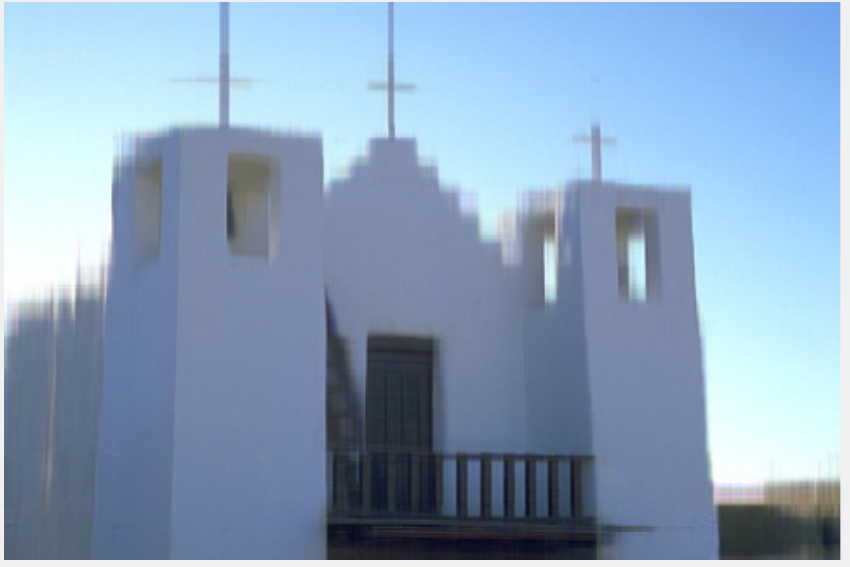}
	\end{subfigure}
	\begin{subfigure}[b]{0.16\textwidth}
		\centering
		\includegraphics[width=\textwidth]{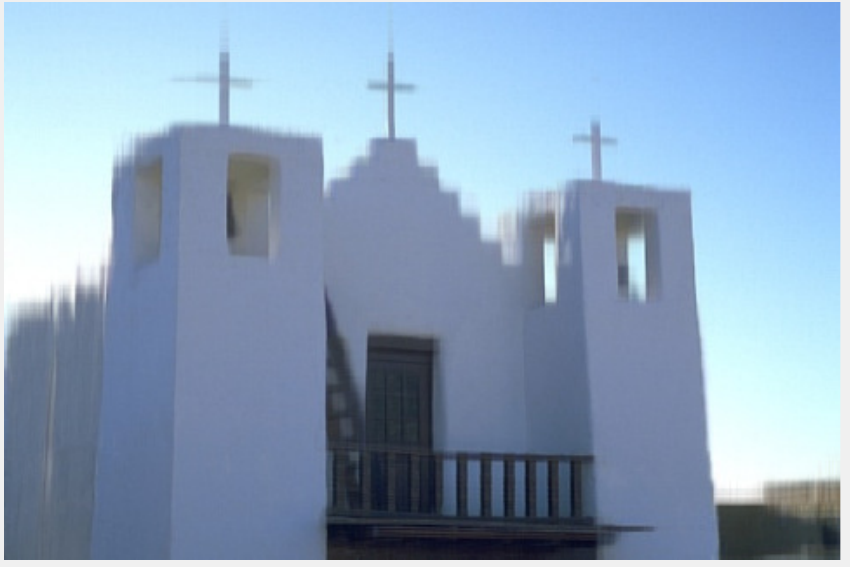}
	\end{subfigure}
	\begin{subfigure}[b]{0.16\textwidth}
		\centering
		\includegraphics[width=\textwidth]{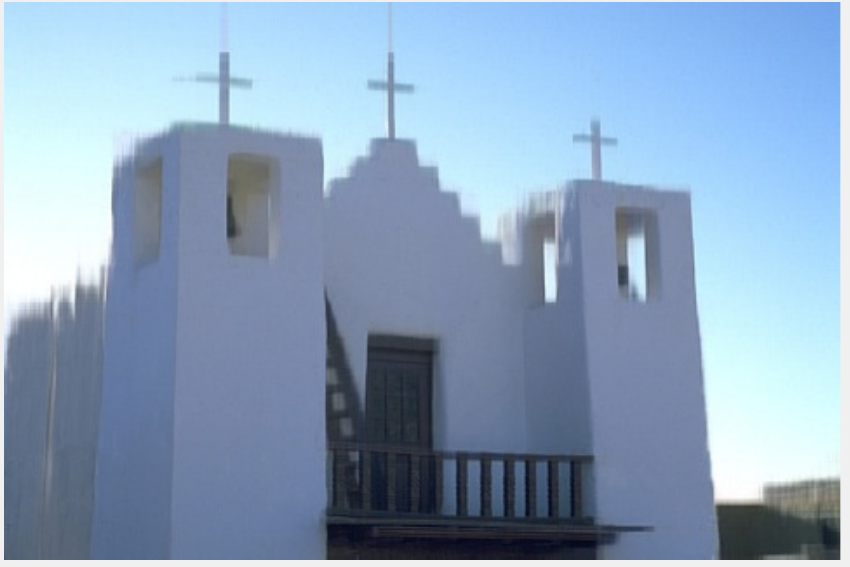}
	\end{subfigure}	
	\begin{subfigure}[b]{0.16\textwidth}
		\centering
		\includegraphics[width=\textwidth]{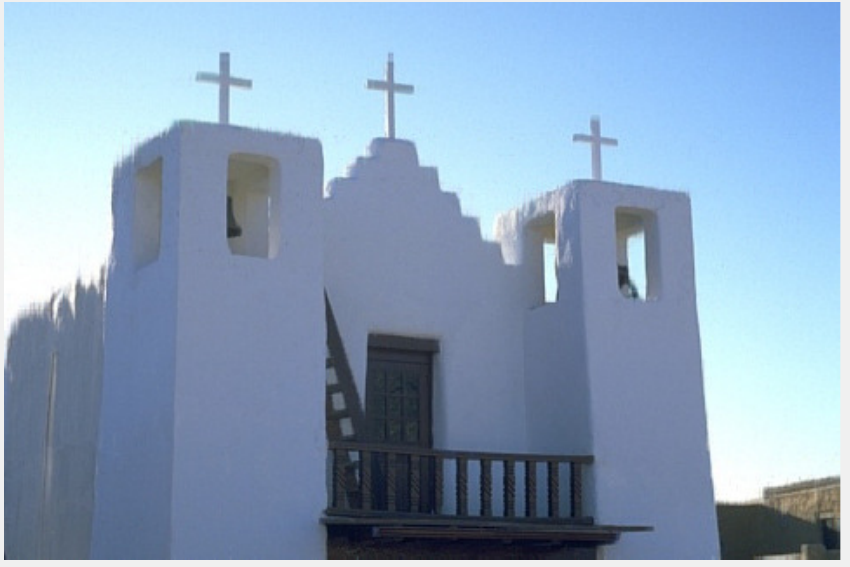}
	\end{subfigure}		
	\\\vspace{0.1em} 
	
	\begin{subfigure}[b]{0.16\textwidth}
		\centering
		\includegraphics[width=\textwidth]{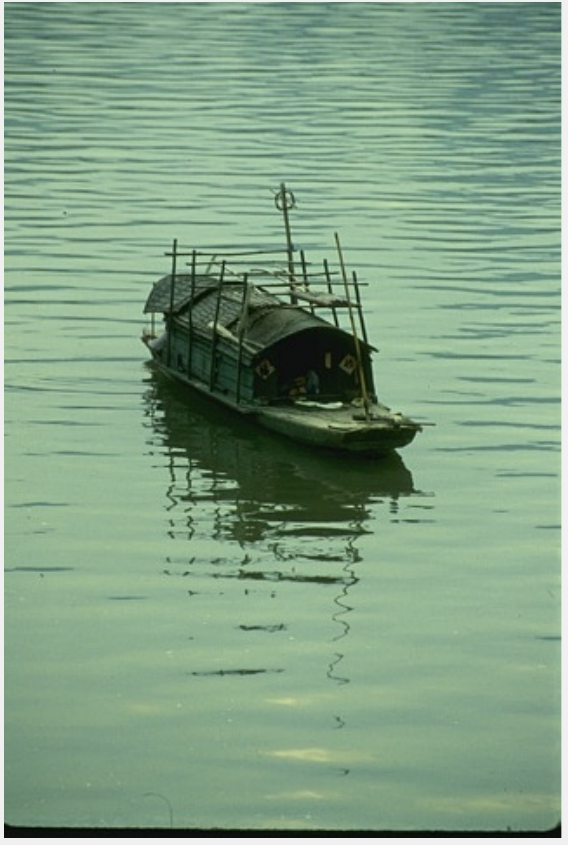}
	\end{subfigure}
	\begin{subfigure}[b]{0.16\textwidth}
		\centering
		\includegraphics[width=\textwidth]{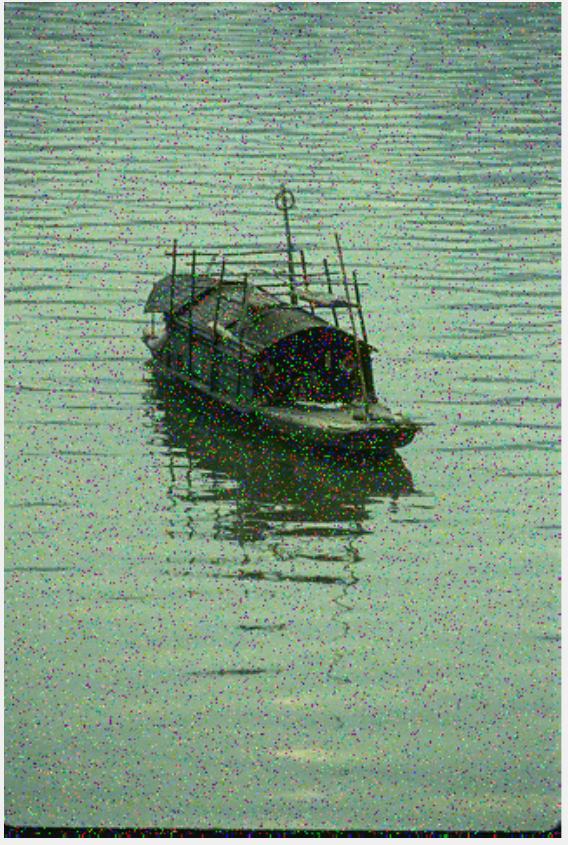}
	\end{subfigure}
	\begin{subfigure}[b]{0.16\textwidth}
		\centering
		\includegraphics[width=\textwidth]{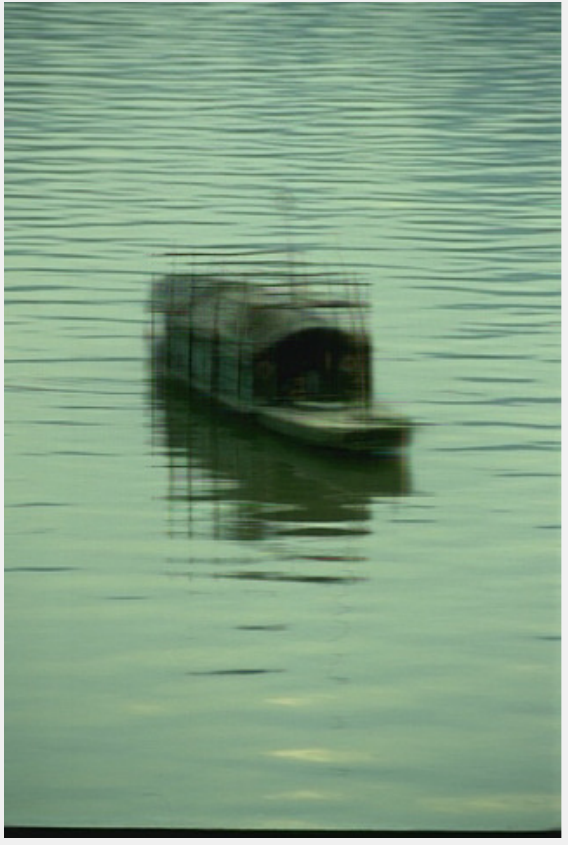}
	\end{subfigure}
	\begin{subfigure}[b]{0.16\textwidth}
		\centering
		\includegraphics[width=\textwidth]{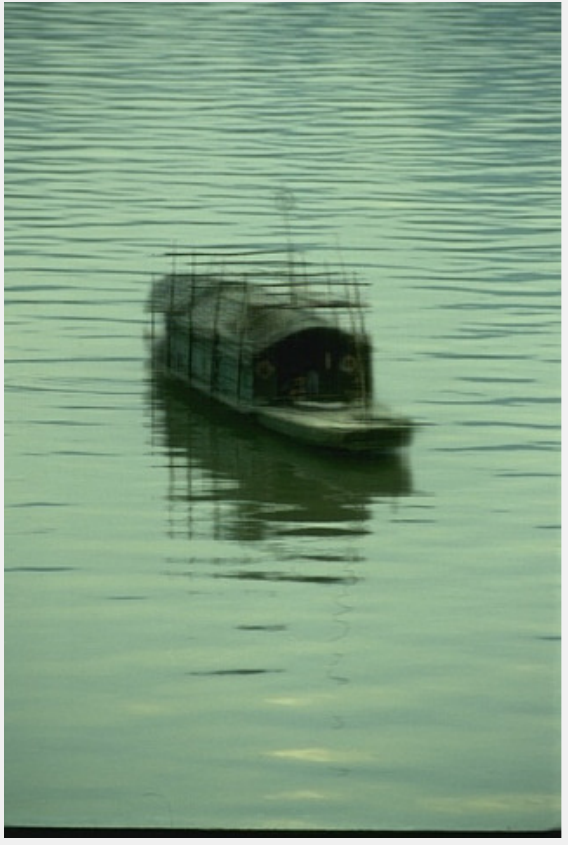}
	\end{subfigure}
	\begin{subfigure}[b]{0.16\textwidth}
		\centering
		\includegraphics[width=\textwidth]{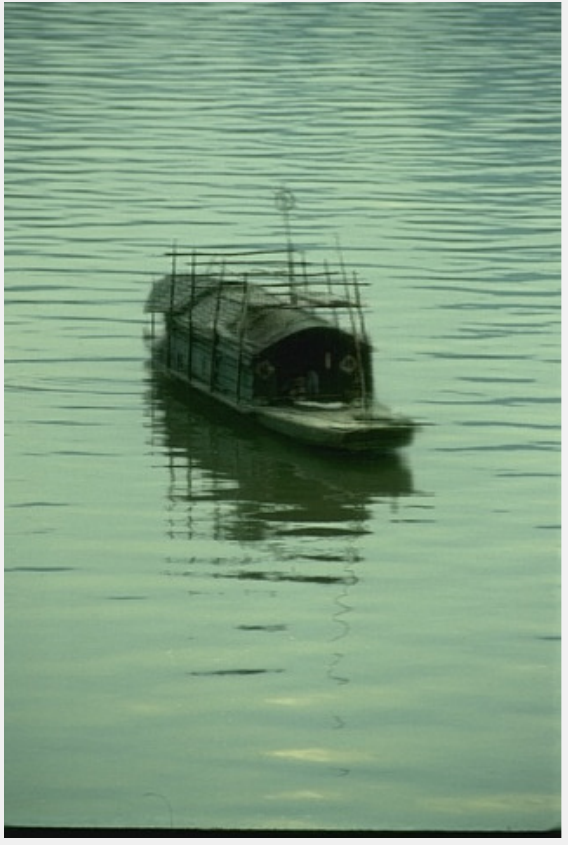}
	\end{subfigure}	
	\begin{subfigure}[b]{0.16\textwidth}
		\centering
		\includegraphics[width=\textwidth]{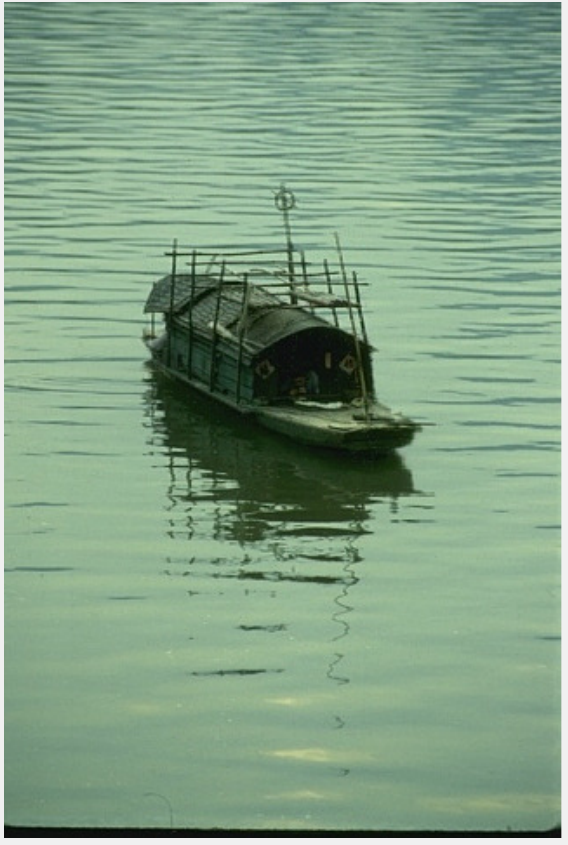}
	\end{subfigure}		
	\\\vspace{0.1em} 
	
	\begin{subfigure}[b]{0.16\textwidth}
		\centering
		\includegraphics[width=\textwidth]{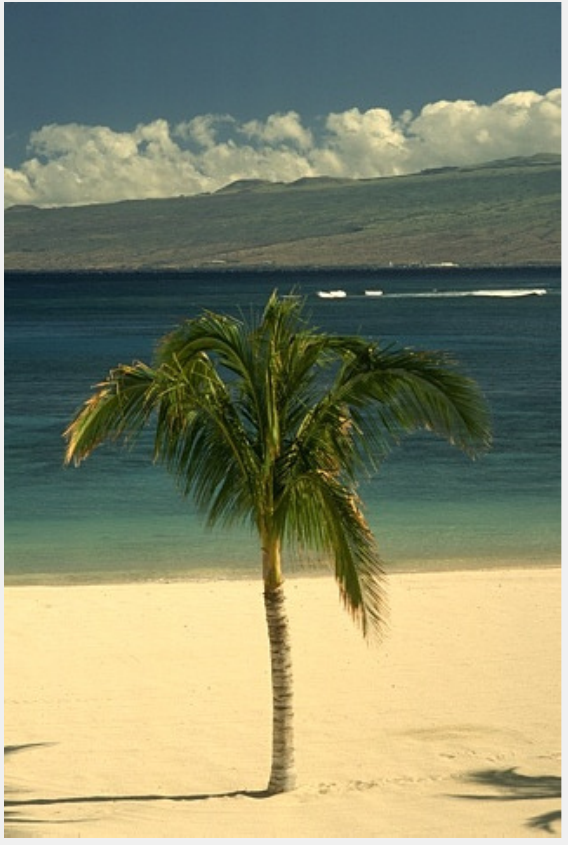}
		\caption{ }
	\end{subfigure}
	\begin{subfigure}[b]{0.16\textwidth}
		\centering
		\includegraphics[width=\textwidth]{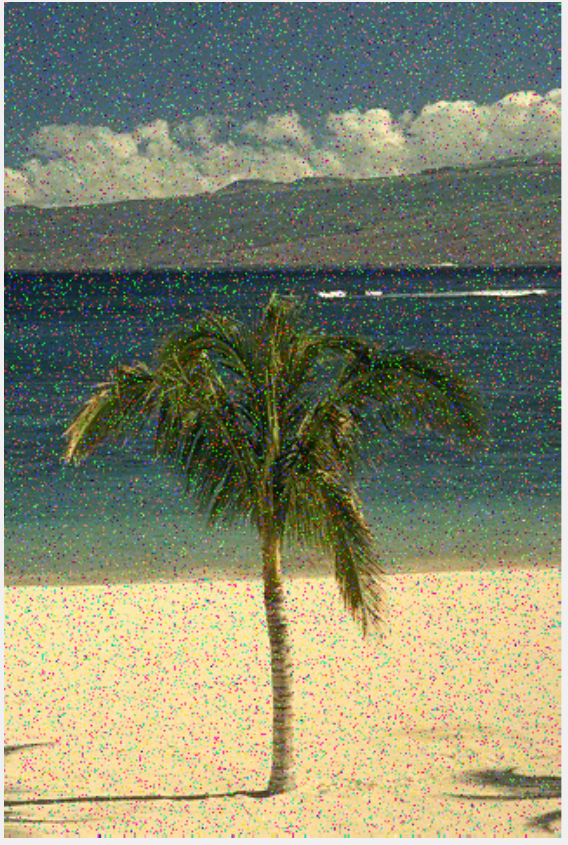}
		\caption{ }
	\end{subfigure}
	\begin{subfigure}[b]{0.16\textwidth}
		\centering
		\includegraphics[width=\textwidth]{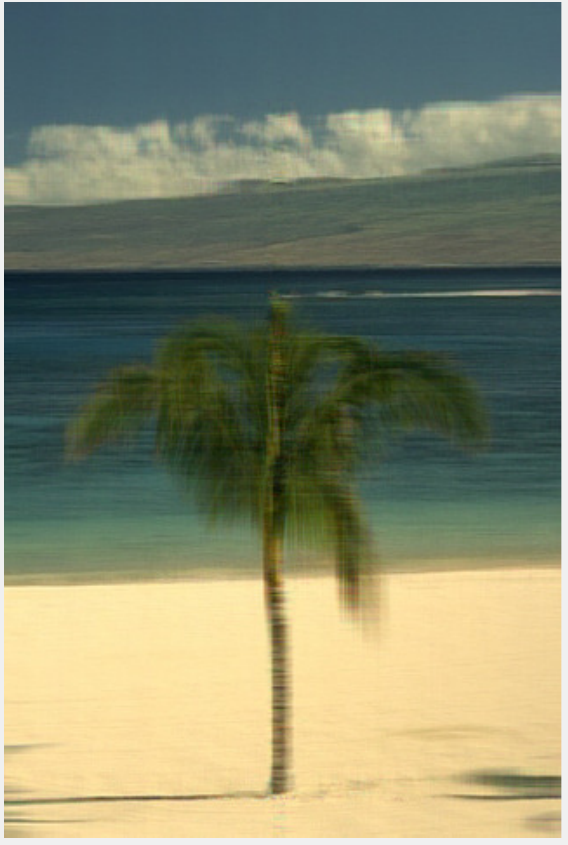}
		\caption{ }
	\end{subfigure}
	\begin{subfigure}[b]{0.16\textwidth}
		\centering
		\includegraphics[width=\textwidth]{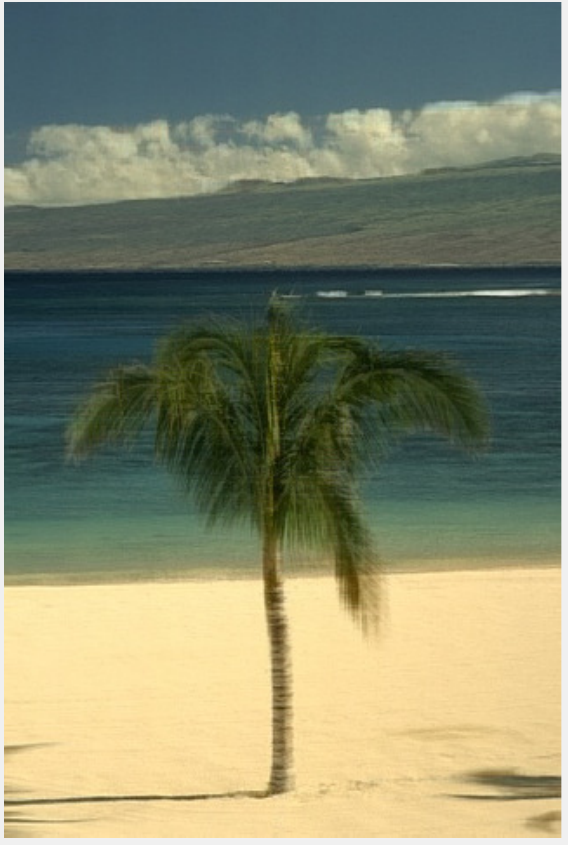}
		\caption{ }
	\end{subfigure}
	\begin{subfigure}[b]{0.16\textwidth}
		\centering
		\includegraphics[width=\textwidth]{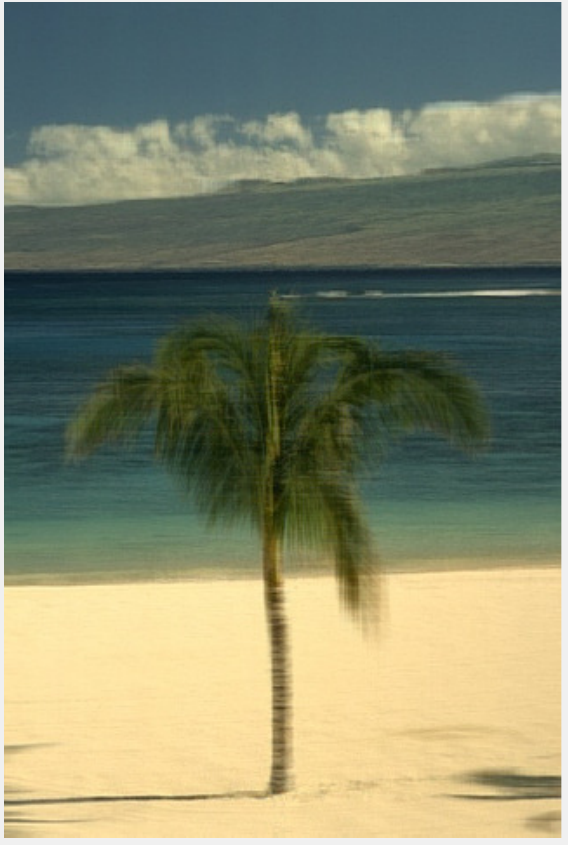}
		\caption{ }
	\end{subfigure}	
	\begin{subfigure}[b]{0.16\textwidth}
		\centering
		\includegraphics[width=\textwidth]{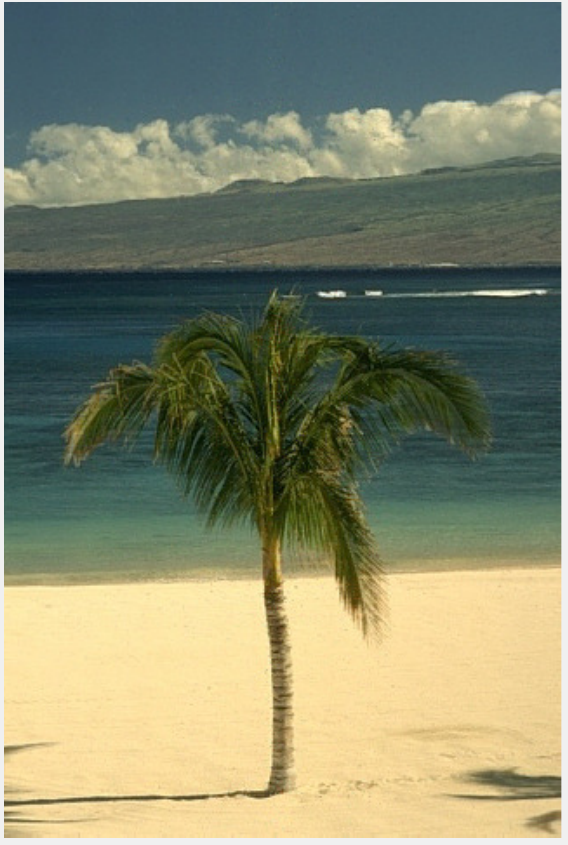}
		\caption{ }
	\end{subfigure}		
	\caption{\small{Performance comparison for image recovery on some sample images. (a) Original image; (b) observed image; (c)-(f) recovered images by RPCA, SNN, TRPCA, and our TRPCA-DCT, respectively. \textbf{Best viewed in $\times 2$ sized color pdf file.}}} \label{fig_imageinpr}
\end{figure*}

\section{Conclusions}\label{sec_con}

Based on the  t-product under invertible linear transforms, we defined the new tensor tubal rank and tensor nuclear norm. We then proposed a new Tensor
Robust Principal Component (TRPCA) model given the transforms satisfying certain conditions. In theory, we proved that, under certain  assumptions, the convex program 
recovers both the low-rank and  sparse components exactly.   Numerical experiments verify our theory and the results on images demonstrate the effectiveness of our model.

This work provides a new direction for the pursuit of low-rank tensor estimation. The tensor rank and tensor nuclear norm  not only depend on the given tensor, but also depend on the used invertible linear transforms. It is obvious that the best transforms will be different for different tasks and different  data. So looking for the optimized transforms is a very interesting future work. Second, though this work focuses on the analysis for 3-way tensors, it may not be difficult to generalize our model and main results to the case of order-$p$ ($p\geq3$) tensors, by using the t-SVD for order-$p$ tensors as in \cite{martin2013order}. Finally, it is always important to apply such a new technique for some other applications.

\section*{Appendix}

\subsection{Proof of Theorem \ref{thmtsvt}}
\begin{proof}
	Let $\Y= \U *_L \SSS *_L\V^*$ be the tensor SVD of $\Y$. By using the definition of tensor nuclear norm and property  (\ref{eq_proFnormNuclear}), problem (\ref{thmeqnsvt}) is equivalent to 
	\begin{align}
	 &\arg\min_{\X} \ \tau \norm{\X}_*+\frac{1}{2}\norm{\X - \Y}_F^2\notag \\
	=&\arg\min_{\X}  \frac{1}{\ell}(\tau\norm{\Xmbar}_*+\frac{1}{2}\norm{\Xmbar-\Ymbar}_F^2) 	\notag \\
	=&\arg\min_{\X}  \ \frac{1}{\ell}\sumi(\tau\norm{\Xmbar^{(i)}}_*+\frac{1}{2}\norm{\Xmbar^{(i)}-\Ymbar^{(i)}}_F^2). \label{eqnthmtsvteqnform}
	\end{align}
	By Theorem 2.1 in \cite{cai2010singular}, we know that the $i$-th frontal slice of $\overline{\mathcal{D}_\tau(\Y)}$ solves the $i$-th subproblem of (\ref{eqnthmtsvteqnform}). Hence, $\mathcal{D}_\tau(\Y)$ solves problem (\ref{thmeqnsvt}).
\end{proof}


{
	\bibliographystyle{IEEEbib}
	\bibliography{TIP_TTRPCA}

\begin{thebibliography}{10}

\bibitem{RPCA}
Emmanuel~J Cand{\`e}s, Xiaodong Li, Yi~Ma, and John Wright,
\newblock ``Robust principal component analysis?,''
\newblock {\em J. ACM}, vol. 58, no. 3, 2011.

\bibitem{lu2016tensorrpca}
Canyi Lu, Jiashi Feng, Yudong Chen, Wei Liu, Zhouchen Lin, and Shuicheng Yan,
\newblock ``Tensor robust principal component analysis: Exact recovery of
  corrupted low-rank tensors via convex optimization,''
\newblock in {\em {Proc. IEEE Conf. Computer Vision and Pattern Recognition}}.
  IEEE, 2016.

\bibitem{kolda2009tensor}
Tamara~G Kolda and Brett~W Bader,
\newblock ``Tensor decompositions and applications,''
\newblock {\em {SIAM Rev.}}, vol. 51, no. 3, pp. 455--500, 2009.

\bibitem{Vasilescu2002Multilinear}
M~Alex~O Vasilescu and Demetri Terzopoulos,
\newblock ``Multilinear analysis of image ensembles: Tensorfaces,''
\newblock in {\em {Proc. European Conf. Computer Vision}}, pp. 447--460.
  Springer, 2002.

\bibitem{franz2009triplerank}
Thomas Franz, Antje Schultz, Sergej Sizov, and Steffen Staab,
\newblock ``Triplerank: Ranking semantic web data by tensor decomposition,''
\newblock in {\em Int'l Semantic Web Conf.}, 2009, pp. 213--228.

\bibitem{sidiropoulos2000parallel}
Nicholas~D Sidiropoulos, Rasmus Bro, and Georgios~B Giannakis,
\newblock ``Parallel factor analysis in sensor array processing,''
\newblock {\em {IEEE Trans. Signal Processing}}, vol. 48, no. 8, pp.
  2377--2388, 2000.

\bibitem{engelen2009a}
Sanne Engelen, Stina Frosch, and Bo~M Jorgensen,
\newblock ``A fully robust parafac method for analyzing fluorescence data,''
\newblock {\em J. Chemometrics}, vol. 23, no. 3, pp. 124--131, 2009.

\bibitem{Acar09scalabletensor}
Evrim Acar, Tamara~G. Kolda, Daniel~M. Dunlavy, and Morten Mørup,
\newblock ``Scalable tensor factorizations with missing data,''
\newblock in {\em {IEEE Int'l Conf. Data Mining}}, 2009, pp. 701--712.

\bibitem{ji2011robust}
Hui Ji, Sibin Huang, Zuowei Shen, and Yuhong Xu,
\newblock ``Robust video restoration by joint sparse and low rank matrix
  approximation,''
\newblock {\em SIAM J. Imaging Sciences}, vol. 4, no. 4, pp. 1122--1142, 2011.

\bibitem{peng2012rasl}
Yigang Peng, Arvind Ganesh, John Wright, Wenli Xu, and Yi~Ma,
\newblock ``{RASL}: Robust alignment by sparse and low-rank decomposition for
  linearly correlated images,''
\newblock {\em {IEEE Trans. Pattern Recognition and Machine Intelligence}},
  vol. 34, no. 11, pp. 2233--2246, 2012.

\bibitem{hillar2013most}
Christopher~J Hillar and Lek-Heng Lim,
\newblock ``Most tensor problems are {NP}-hard,''
\newblock {\em J. ACM}, vol. 60, no. 6, pp. 45, 2013.

\bibitem{liu2013tensor}
Ji~Liu, Przemyslaw Musialski, Peter Wonka, and Jieping Ye,
\newblock ``Tensor completion for estimating missing values in visual data,''
\newblock {\em {IEEE Trans. Pattern Recognition and Machine Intelligence}},
  vol. 35, no. 1, pp. 208--220, 2013.

\bibitem{gandy2011tensor}
Silvia Gandy, Benjamin Recht, and Isao Yamada,
\newblock ``Tensor completion and low-n-rank tensor recovery via convex
  optimization,''
\newblock {\em Inverse Problems}, vol. 27, no. 2, pp. 025010, 2011.

\bibitem{tomioka2011statistical}
Ryota Tomioka, Taiji Suzuki, Kohei Hayashi, and Hisashi Kashima,
\newblock ``Statistical performance of convex tensor decomposition,''
\newblock in {\em {Advances in Neural Information Processing Systems}}, 2011,
  pp. 972--980.

\bibitem{kreimer2013nuclear}
Nadia Kreimer and Mauricio~D Sacchi,
\newblock ``Nuclear norm minimization and tensor completion in exploration
  seismology,''
\newblock in {\em {Proc. IEEE Int'l Conf. Acoustics, Speech and Signal
  Processing}}, 2013, pp. 4275--4279.

\bibitem{signoretto2011tensor}
Marco Signoretto, R~Van De~Plas, B~De~Moor, and Johan A~K Suykens,
\newblock ``Tensor versus matrix completion: A comparison with application to
  spectral data,''
\newblock {\em IEEE Signal Processing Letters}, vol. 18, no. 7, pp. 403--406,
  2011.

\bibitem{mu2013square}
Cun Mu, Bo~Huang, John Wright, and Donald Goldfarb,
\newblock ``Square deal: Lower bounds and improved relaxations for tensor
  recovery,''
\newblock in {\em {Proc. Int'l Conf. Machine Learning}}, 2014, pp. 73--81.

\bibitem{huang2014provable}
Bo~Huang, Cun Mu, Donald Goldfarb, and John Wright,
\newblock ``Provable models for robust low-rank tensor completion,''
\newblock {\em Pacific J. Optimization}, vol. 11, no. 2, pp. 339--364, 2015.

\bibitem{candes2009exact}
Emmanuel~J Cand{\`e}s and Benjamin Recht,
\newblock ``Exact matrix completion via convex optimization,''
\newblock {\em Foundations of Computational mathematics}, vol. 9, no. 6, pp.
  717--772, 2009.

\bibitem{romera2013new}
Bernardino Romera-Paredes and Massimiliano Pontil,
\newblock ``A new convex relaxation for tensor completion,''
\newblock in {\em {Advances in Neural Information Processing Systems}}, 2013,
  pp. 2967--2975.

\bibitem{kilmer2011factorization}
Misha~E Kilmer and Carla~D Martin,
\newblock ``Factorization strategies for third-order tensors,''
\newblock {\em Linear Algebra and its Applications}, vol. 435, no. 3, pp.
  641--658, 2011.

\bibitem{zhang2014novel}
Zemin Zhang, Gregory Ely, Shuchin Aeron, Ning Hao, and Misha Kilmer,
\newblock ``Novel methods for multilinear data completion and de-noising based
  on tensor-{SVD},''
\newblock in {\em {Proc. IEEE Conf. Computer Vision and Pattern Recognition}}.
  IEEE, 2014, pp. 3842--3849.

\bibitem{semerci2014tensor}
Oguz Semerci, Ning Hao, Misha~E Kilmer, and Eric~L Miller,
\newblock ``Tensor-based formulation and nuclear norm regularization for
  multienergy computed tomography,''
\newblock {\em {IEEE Trans. Image Processing}}, vol. 23, no. 4, pp. 1678--1693,
  2014.

\bibitem{lu2018exact}
Canyi Lu, Jiashi Feng, Zhouchen Lin, and Shuicheng Yan,
\newblock ``Exact low tubal rank tensor recovery from {G}aussian
  measurements,''
\newblock in {\em {Int'l Joint Conf. Artificial Intelligence}}, 2018.

\bibitem{lu2018tensor}
Canyi Lu, Jiashi Feng, Yudong Chen, Wei Liu, Zhouchen Lin, and Shuicheng Yan,
\newblock ``Tensor robust principal component analysis with a new tensor
  nuclear norm,''
\newblock {\em {IEEE Trans. Pattern Recognition and Machine Intelligence}},
  2019.

\bibitem{chen2013incoherence}
Yudong Chen,
\newblock ``Incoherence-optimal matrix completion,''
\newblock {\em {IEEE Trans. Information Theory}}, vol. 61, no. 5, pp.
  2909--2923, May 2015.

\bibitem{zhang2015exact}
Zemin Zhang and Shuchin Aeron,
\newblock ``Exact tensor completion using t-{SVD},''
\newblock {\em {IEEE Trans. Signal Processing}}, vol. 65, no. 6, pp.
  1511--1526, 2017.

\bibitem{kernfeld2015tensor}
Eric Kernfeld, Misha Kilmer, and Shuchin Aeron,
\newblock ``Tensor--tensor products with invertible linear transforms,''
\newblock {\em Linear Algebra and its Applications}, vol. 485, pp. 545--570,
  2015.

\bibitem{gleich2013power}
David~F Gleich, Chen Greif, and James~M Varah,
\newblock ``The power and arnoldi methods in an algebra of circulants,''
\newblock {\em Numerical Linear Algebra with Applications}, vol. 20, no. 5, pp.
  809--831, 2013.

\bibitem{lu2018unified}
Canyi Lu, Jiashi Feng, Shuicheng Yan, and Zhouchen Lin,
\newblock ``A unified alternating direction method of multipliers by
  majorization minimization,''
\newblock {\em {IEEE Trans. Pattern Recognition and Machine Intelligence}},
  vol. 40, no. 3, pp. 527--541, 2018.

\bibitem{martin2001database}
David Martin, Charless Fowlkes, Doron Tal, and Jitendra Malik,
\newblock ``A database of human segmented natural images and its application to
  evaluating segmentation algorithms and measuring ecological statistics,''
\newblock in {\em {Proc. IEEE Int'l Conf. Computer Vision}}. IEEE, 2001,
  vol.~2, pp. 416--423.

\bibitem{martin2013order}
Carla~D Martin, Richard Shafer, and Betsy LaRue,
\newblock ``An order-$p$ tensor factorization with applications in imaging,''
\newblock {\em SIAM J. Scientific Computing}, vol. 35, no. 1, pp. A474--A490,
  2013.

\bibitem{cai2010singular}
Jianfeng Cai, Emmanuel Cand{\`e}s, and Zuowei Shen,
\newblock ``A singular value thresholding algorithm for matrix completion,''
\newblock {\em SIAM J. Optimization}, 2010.

\end{thebibliography}
}


\vspace{-380pt}

\begin{IEEEbiography}{Canyi Lu} is now a postdoctoral research associate in the Carnegie Mellon University. He received
	his Ph.D. degree from the National University of
	Singapore in 2017. His current research interests include computer vision, machine learning,
	pattern recognition and optimization. He was
	the winner of 2014 Microsoft Research Asia
	Fellowship, 2017 Chinese Government Award
	for Outstanding Self-Financed Students Abroad
	and Best Poster Award in BigDIA 2018. 
\end{IEEEbiography}
\vspace{-350pt}
\begin{IEEEbiography}{Pan Zhou}
received Master Degree in computer
science from Peking University in 2016. Now he
is a Ph.D. candidate at the Department of Electrical and Computer Engineering (ECE), National
University of Singapore, Singapore. His research
interests include computer vision, machine learning,
and optimization. He was
the winner of 2018 Microsoft Research Asia
Fellowship.
\end{IEEEbiography}

\end{document}